\documentclass[sigconf]{acmart}

\AtBeginDocument{%
  }

\usepackage{url}            
\usepackage{amsfonts}       

\usepackage{algorithm}
\usepackage[noend]{algpseudocode}
\usepackage{mathtools}
\usepackage{tikz}
\usepackage{tikz-3dplot}
\usepackage{pgfplots}
\usepgfplotslibrary{ternary}
\usetikzlibrary{calc,shapes}
\usepackage{thmtools}
\usepackage{thm-restate}
\usepackage{multirow}
\usepackage{array}
\usepackage{tabu}

\newtheorem{assumption}{Assumption}

\newcommand{\alg}{\textsf{GCB$_{\mathtt{safe}}$}}
\newcommand{\opt}{\textsf{Opt}$(\boldsymbol{\mu}, \lambda)$}

\newcommand{\bi}{x}

\newcommand{\co}{c}
\newcommand{\realbi}{\hat{x}}

\usepackage{eurosym}
\usepackage{bm}

\usepackage{subfig} 
\pgfplotsset{compat=1.18}

\allowdisplaybreaks

\begin{document}

\title{Safe Online Bid Optimization with Return on Investment and Budget Constraints}


\author{Matteo Castiglioni}
\email{matteo.castiglioni@polimi.it}
\affiliation{%
	\institution{Politecnico di Milano}
	\city{Milano}
	\country{Italy}
}


\author{Alessandro Nuara}
\email{alessandro.nuara@adcube.com}
\affiliation{%
	\institution{ADcube, Politecnico di Milano}
	\city{Milano}
	\country{Italy}
}

\author{Giulia Romano}
\email{giulia.romano@polimi.it}
\affiliation{%
  \institution{Politecnico di Milano}
  \city{Milano}
  \country{Italy}
}

\author{Giorgio Spadaro}
\email{giorgio.spadaro@mail.polimi.it}
\affiliation{%
	\institution{Politecnico di Milano}
	\city{Milano}
	\country{Italy}
}

\author{Francesco Trovò}
\email{francesco1.trovo@polimi.it}
\affiliation{%
\institution{Politecnico di Milano}
\city{Milano}
\country{Italy}
}

\author{Nicola Gatti}
\email{nicola.gatti@polimi.it}
\affiliation{%
\institution{Politecnico di Milano}
\city{Milano}
\country{Italy}
}

\renewcommand{\shortauthors}{Matteo Castiglioni et al.}




\begin{abstract}
In online marketing, the advertisers aim to balance achieving \emph{high volumes} and \emph{high profitability}.
The companies' business units address this tradeoff by maximizing the volumes while guaranteeing a minimum Return On Investment (ROI) level.
Such a task can be naturally modeled as a combinatorial optimization problem subject to ROI and budget constraints that can be solved online.
In this picture, the learner's uncertainty over the constraints' parameters plays a crucial role since the algorithms' exploration choices might lead to their violation during the entire learning process.
Such violations represent a major obstacle to adopting online techniques in real-world applications.
Thus, controlling the algorithms' exploration during learning is paramount to making humans trust online learning tools.
This paper studies the nature of both optimization and learning problems.
In particular, we show that the learning problem is inapproximable within any factor (unless $\textsf{P} = \textsf{NP}$) and provide a pseudo-polynomial-time algorithm to solve its discretized version.
Subsequently, we prove that no online learning algorithm can violate the (ROI or budget) constraints a sublinear number of times during the learning process while guaranteeing a sublinear regret.
We provide the \textsf{GCB} algorithm that guarantees sublinear regret at the cost of a linear number of constraint violations and \alg{} that guarantees w.h.p.~a constant upper bound on the number of constraint violations at the cost of a linear regret.
Moreover, we designed \alg{}$(\psi,\phi)$, which guarantees both sublinear regret and safety w.h.p.~at the cost of accepting tolerances $\psi$ and $\phi$ in the satisfaction of the ROI and budget constraints, respectively.
Finally, we provide experimental results to compare the regret and constraint violations of \textsf{GCB}, \alg{}, and \alg{}$(\psi,\phi)$.
\end{abstract}

\maketitle

\section{Introduction}

Nowadays, Internet advertising is \emph{de facto} the leading advertising medium.
Notably, while the expenditure on physical ads, radio, and television has been stable for a decade, that on Internet advertising is increasing with a ratio of $\approx 13\%$ per year, reaching the amount of $225$ billion USD in $2023$ only in the US~\cite{iab2024iab}.
Internet advertising has two main advantages over traditional channels: it provides precise ad targeting and allows an accurate, real-time evaluation of investment performance.
On the other hand, the amount of data the platforms provide and the plethora of parameters to be set make its optimization impractical without using artificial intelligence.

The advertisers' goal is usually to set \emph{bids} in the attempt to balance the tradeoff between achieving \emph{high volumes}, corresponding to maximizing the sales of the products to advertise, and \emph{high profitability}, corresponding to maximizing Return On Investment (ROI).
The companies' business units need simple ways to address this tradeoff, and customarily, they maximize the volumes while constraining ROI to be above a given threshold.
In addition to standard budget constraints (i.e., spent per day on advertising), the importance of ROI constraints has been highlighted in several empirical studies.
We mention, \emph{e.g.}, the data analysis on the auctions performed on Google's AdX in~\cite{golrezaei2021auction}, showing that many advertisers take into account ROI constraints, particularly in hotel booking.
However, no platform provides features to force the satisfaction of ROI constraints, e.g., in Meta, the constraints we might specify are in terms of ROI (called ROAS) for single subcampaigns.\footnote{\url{https://www.facebook.com/business/help/1113453135474912?id=2196356200683573}.}
Notably, some platforms (\emph{e.g.}, TripAdvisor, and Trivago) do not even allow the setting of daily budget constraints.
Thus, the problem of satisfying those constraints is a challenge that advertisers must address by designing suitable bidding strategies.
In this picture, uncertainty plays a crucial role as the revenue and cost of the advertising campaigns are unknown beforehand and need to be estimated online by learning algorithms during the sequential arrival of data.
As a result, the constraints are subject to uncertainty, and wrong parameter estimations can arbitrarily violate the ROI and budget constraints when using an uncontrolled exploration like that used by classical online learning algorithms.
Such violations represent today the major obstacles to adopting AI tools in real-world applications, as advertisers often are unwilling to accept such risks.
Remarkably, this issue is particularly crucial in the early stages of the learning process as adopting algorithms with an uncontrolled exploration when a small amount of data is available can make the advertising campaigns' performance oscillate by a large magnitude, which advertisers negatively perceive. 
Therefore, to make humans trust online artificial intelligence algorithms, controlling their exploration accurately to mitigate risks and provide safety guarantees during the learning process is paramount.

\paragraph{Original Contributions}
As customary in the online advertising literature, see, \emph{e.g.}, \citet{devanur2009price}, we assume stochastic (\emph{i.e.}, non-adversarial) clicks, and we adopt Gaussian Processes (GPs) to model the problem.
The assumption that clicks are generated stochastically is reasonable in practice as the advertising platforms can limit manipulation due to malicious bidders.
For instance, Google Ads can identify invalid clicks and exclude them from the advertisers' spending.
This paper studies the nature of both the optimization and learning problems.
Initially, we focus on studying our optimization problem without uncertainty, showing that no approximation within any strictly positive factor is possible with ROI and budget constraints unless $\mathsf{P} = \mathsf{NP}$, even in simple, realistic instances.
However, when dealing with a discretized space of the bids as it happens in practice, the problem admits an exact pseudo-polynomial time algorithm based on dynamic programming.
Most importantly, when the problem is with uncertainty, we show that no online learning algorithm can violate the ROI and/or budget constraints a sublinear number of times while guaranteeing a sublinear pseudo-regret.
We provide an online learning algorithm, namely \textsf{GCB}, providing pseudo-regret sublinear in the time horizon $T$ at the cost of a linear number of violations of the constraints.
We also provide its safe version, namely \alg{}, guaranteeing w.h.p.~a constant upper bound on the number of constraints' violations at the cost of a regret linear in $T$. 
Inspired by the two previous algorithms, we design a new algorithm, namely \alg$(\psi, \phi)$, guaranteeing both the violation w.h.p.~of the constraints for a constant number of times and a pseudo-regret sublinear both in $T$ and the maximum information gain of the GP when accepting tolerances $\psi$ and $\phi$ in the satisfaction of the ROI and budget constraints, respectively.
Finally, we performed an empirical study of our algorithms in terms of pseudo-regret/number of constraint violations tradeoff in simulated settings, showing the importance of adopting safety constraints and the effectiveness of our algorithms.

\section{Related Works} \label{app:relworks}

Many works study Internet advertising, both from the \emph{publisher} perspective (\emph{e.g.}, \citet{vazirani2007algorithmic} design auctions for ads allocation and pricing) and from the \emph{advertiser} perspective (\emph{e.g.}, \citet{feldman2007budget} study the budget optimization problem in search advertising).
Only a few works deal with ROI constraints, and, to the best of our knowledge, they focus on the design of auction mechanisms.
In particular, \citet{szymanski2006impact} and~\citet{borgs2007dynamics} show that ROI-based bidding heuristics can lead to cyclic behavior and reduce the allocation's efficiency, while~\citet{golrezaei2021auction} propose more efficient auctions with ROI constraints.
The learning algorithms for daily bid optimization available in the literature address only budget constraints in the restricted case in which the platform allows the advertisers to set a daily budget limit (notice that some platforms, such as TripAdvisor and Trivago, do not even allow the setting of the daily budget limit).
For instance, \citet{zhang2012joint} provide an \emph{offline} algorithm that exploits accurate models of the campaigns' performance based on low-level data rarely available to the advertisers.
\citet{feng2023online} propose an online auto-bidding algorithm for a single advertiser maximizing value under the Return-on-Spend constraint. However, the constraint violation is evaluated on the cumulated value of the violation over the entire time horizon (which is a condition much weaker than ours).
\citet{nuara2018combinatorial} propose an \emph{online} learning algorithm that combines combinatorial multi-armed bandit techniques~\cite{chen2013combinatorial} with regression by Gaussian Processes~\cite{rasmussen2006gaussian}.
This work provides no guarantees on ROI.
More recent works also present pseudo-regret bounds~\cite{nuara2022online} and study subcampaigns interdependencies offline~\citep{nuara2019dealing}.
\citet{thomaidou2014toward} provide a genetic algorithm for budget optimization of advertising campaigns.
\citet{ding2013multi} and~\citet{trovo2016budgeted} address the bid optimization problem in a single subcampaign scenario when the budget constraint is cumulative over time.
Recently, \citet{deng2023multi} analysed the problem of multichannel optimization with ROI constraints, where the guarantees are provided on average (where our work is considering per-round constraint violations).

A research field strictly related to our work is learning with safe exploration and constraints subject to uncertainty.
The goal is to guarantee w.h.p.~the constraints satisfaction during the entire learning process.
The only known results on safe exploration in multi-armed bandits address the case with continuous, convex arm spaces and convex constraints.
The learner can converge to the optimal solution in these settings without violating the constraints~\citep{moradipari2020stagewise,amani2020regret}.
Conversely, the case with discrete and/or non-convex arm spaces or non-convex constraints, such as ours, is unexplored in the literature.
We remark that some bandit algorithms address uncertain constraints where the goal is their satisfaction on average~\cite{mannor2009online,cao2019online}.  
However, the per-round violation can be arbitrarily large (particularly in the early stages of the learning process), and it would not fit with our setting as humans could be alarmed and, thus, give up on adopting the algorithm.
Moreover, several works in reinforcement learning~\citep{hans2008safe,pirotta2013safe,garcia2012safe} and multi-armed bandit~\citep{galichet2013exploration,sui2015safe} investigate safe exploration, providing safety guarantees on the revenue provided by the algorithm, but not on the satisfaction w.h.p.~of uncertain constraints.

Another line of research dealing with bandits with constraints is the one related to Combinatorial MAB with knapsack or budget constraints, e.g., the works by~\cite{sankararaman2018combinatorial,badanidiyuru2018bandits,balseiro2019learning,balseiro2020dual,feng2023online,lucier2024autobidders}.
However, they cannot be applied to our setting for one of two main reasons. Some works, i.e., \cite{sankararaman2018combinatorial,badanidiyuru2018bandits} provide theoretical guarantees in the case the decision sets are matroidal-shaped. However, the set of arms in our setting is non-matroidal.
Others, i.e., \cite{balseiro2019learning,balseiro2020dual,feng2023online,lucier2024autobidders}, guarantee the satisfaction of the budget constraints over a given time horizon. That is, the bandit algorithm cannot spend more than a given threshold within the given time horizon. Conversely, in our work, we bound, with high probability, the violation of the constraint in every round.

Finally, we also mentioned a line of research loosely related to ours about the analysis of the ROI constraints in the context of auction design~\cite{golrezaei2021bidding}.

\section{Optimization Problem} \label{sec:2}

We are given an advertising campaign $\mathcal{C} = \{ C_1, \ldots, C_N \}$, with $N \in \mathbb{N}$ and where $C_j$ is the $j$-th subcampaign, and a finite time horizon of $T \in \mathbb{N}$ rounds (each corresponding to one day in our application).
In this work, as common in the literature on ad allocation optimization, we refer to a subcampaign as a single ad or a group of homogeneous ads requiring setting the same bid.
For every round $t \in \{1, \ldots, T\}$ and every subcampaign $C_j$, an advertiser needs to specify the bid $\bi_{j,t} \in X_j$, where $X_j \subset \mathbb{R}^+$ is a finite set of values that can be set for subcampaign $C_j$.
For every round $t \in \{1, \ldots, T\}$, the goal is to find the values of bids maximizing the overall cumulative expected revenue while keeping the ROI above a fixed value $\lambda \in \mathbb{R}^+$ and the budget below a daily value $\beta \in \mathbb{R}^+$.
Formally, the resulting constrained optimization problem at round $t$ is as follows:
\begin{subequations}
	\begin{align}
		\max_{(\bi_{1,t}, \ldots, \bi_{N,t}) \in X_1 \times \ldots \times X_N} & \sum_{j=1}^N v_j \ n_j(\bi_{j,t}) \label{formulation:objectivefunction} \\
		\text{s.t.} \qquad \qquad &\frac{\sum_{j=1}^N v_j \ n_j(\bi_{j,t}) }{\sum_{j=1}^N \ \co_j(\bi_{j,t}) } \geq \lambda, \label{formulation:roiconstraint} \\
		& \sum_{j=1}^{N} c_j(x_{j,t}) \leq \beta, \label{formulation:budgetconstraint}
	\end{align}
\end{subequations}
where $n_j(\bi_{j,t})$ and $\co_j(\bi_{j,t})$ are the expected number of clicks and the expected cost given the bid $\bi_{j,t}$ for subcampaign $C_j$, respectively, and $v_j$ is the value per click for subcampaign $C_j$.
We remark that Constraint~(\ref{formulation:roiconstraint}) is the ROI constraint, forcing the revenue to be at least $\lambda$ times the costs, and Constraint~(\ref{formulation:budgetconstraint}) keeps the daily spend under a predefined overall budget $\beta$.\footnote{
	In the economic literature, it is also used as an alternative definition of ROI: $\frac{\sum_{j=1}^N \left[ v_j \ n_j(\bi_{j,t}) - \co_j(\bi_{j,t}) \right]}{\sum_{j=1}^N \ \co_j(\bi_{j,t}) }$.
	We can capture this case by substituting the right-hand side of Constraint~\eqref{formulation:roiconstraint} with $\lambda+1$.}

At first, we show that, even if all the values of the parameters of the optimization problem are known, the optimal solution cannot be approximated in polynomial time within any strictly positive factor (even depending on the size of the instance) unless $\mathsf{P} = \mathsf{NP}$.
\begin{restatable}[Inapproximability]{thm}{thminapprox} \label{thm:inapprox}
	For any $\rho \in (0,1]$, there is no polynomial-time algorithm returning a $\rho$-approximation to the problem in Equations~(\ref{formulation:objectivefunction})-(\ref{formulation:budgetconstraint}), unless $\mathsf{P} = \mathsf{NP}$.
\end{restatable}
The proof follows from a reduction from SUBSET-SUM that is an $\mathsf{NP}$-hard problem.\footnote{Given a set $S$ of integers $u_i \in \mathbb{N}^+$ and an integer $z \in \mathbb{N}^+$, SUBSET-SUM requires to decide whether there is a set $S^*\subseteq S$ with $\sum_{i \in S^*} u_i = z$.}{$^{,}$}\footnote{
	For reasons of space, the proofs are deferred to the supplementary material.}
It is well known that SUBSET-SUM is a weakly $\mathsf{NP}$-hard problem, admitting an exact algorithm whose running time is polynomial in the size of the problem and the magnitude of the data involved rather than the base-two logarithm of their magnitude.
The same can be shown for our problem.
Indeed, we can design a pseudo-polynomial-time algorithm to find the optimal solution in polynomial time w.r.t.~the number of possible values of revenues and costs.
In real-world settings, the values of revenue and cost are in limited ranges and rounded to the nearest cent, allowing the problem to be solved in a reasonable time.
Therefore, in what follows, we assume we are given a discretization of the daily costs and revenue value ranges.

\section{Online Learning Problem Formulation} \label{sec:3}

We focus on the case in which $n_j(\cdot)$ and $c_j(\cdot)$ in Equations~(\ref{formulation:objectivefunction})-(\ref{formulation:budgetconstraint}) are unknown functions whose values need to be estimated online.
This goal is achieved using noisy realizations:
\begin{align*}
	\tilde{n}_{j,h}(\bi_{j,t}) = n_j(\bi_{j,t}) + \zeta_n\\
	\tilde{c}_{j,h}(\bi_{j,t}) = c_j(\bi_{j,t}) + \zeta_c
\end{align*}
obtained setting the bid $\bi_{j,t}$ on the $j$-th subcampaign at time $h$, where $\zeta_n$ and $\zeta_c$ are Gaussian zero-mean noise with variance $\sigma_n^2 \in \mathbb{R}^+$ and $\sigma_c^2 \in \mathbb{R}^+$, respectively.
Our problem can be naturally modeled as a multi-armed bandit where the available \emph{arms} are the different values of the bid $\bi_{j,t} \in X_j$ satisfying the combinatorial constraints of the optimization problem.\footnote{
	Here, we assume that the value per click $v_j$ is known.
	Refer to~\citet{nuara2018combinatorial} to deal the case in which $v_j$ is unknown.}
More specifically, our goal is to design a \emph{learning policy} $\mathfrak{U}$ returning, at every round $t$, a \emph{super-arm} $\left\{ \realbi_{j,t} \right\}_{j=1}^N$, i.e., an arm profile specifying one bid per subcampaign. Since the policy $\mathfrak{U}$ can only use estimates of the unknown number-of-click and cost functions built using past realizations $\tilde{n}_{j,h}(\bi_{j,t})$ and $\tilde{c}_{j,h}(\bi_{j,t})$ with $h < t$, the solutions returned by policy $\mathfrak{U}$ may not be optimal and/or violate Constraints~(\ref{formulation:roiconstraint}) and~(\ref{formulation:budgetconstraint}) when evaluated with the true values.
Therefore, we must design $\mathfrak{U}$ so that the violations occur only for a limited number of rounds over the time horizon $T$.

We are interested in evaluating learning policies $\mathfrak{U}$ in terms of revenue losses (a.k.a.~pseudo-regret) and safety regarding ROI and budget constraints violations.
\begin{restatable}[Learning policy pseudo-regret]{defi}{regret}
	Given a learning policy $\mathfrak{U}$, the \emph{pseudo-regret} is:
	\begin{equation*}
		R_T(\mathfrak{U}) := T \,G^* - \mathbb{E} \left[ \sum_{t = 1}^T \sum_{j=1}^N v_j \ n_j(\realbi_{j,t}) \right],
	\end{equation*}
	where $G^* := \sum_{j=1}^N v_j \ n_j(\bi^*_j)$ is the expected revenue provided by a clairvoyant algorithm, the set of bids $\left\{ \bi^*_j \right\}_{j=1}^N$ is the optimal clairvoyant solution to the problem in Equations~(\ref{formulation:objectivefunction})-(\ref{formulation:budgetconstraint}), and the expectation $\mathbb{E}[\,\cdot\,]$ is taken w.r.t.~the stochasticity of $\mathfrak{U}$.
\end{restatable}
\begin{restatable}[$\eta$-safe learning policy]{defi}{regret}
	Given $\eta \in (0, \ T]$, a learning policy $\mathfrak{U}$ generating an allocation $\left\{ \realbi_{j,t} \right\}_{j=1}^N$ is $\eta$-\emph{safe} if the expected number of times it violates at least one of the Constraints~(\ref{formulation:roiconstraint}) and~(\ref{formulation:budgetconstraint}) from $t=1$ to $t=T$ is less than $\eta$ or, formally:
	\begin{equation*}
		\sum_{t=1}^T \mathbb{P} \left( \frac{\sum_{j=1}^N v_j \ n_j(\realbi_{j,t}) }{\sum_{j=1}^N \ \co_j(\realbi_{j,t}) } < \lambda \vee \sum_{j=1}^{N} c_j(\realbi_{j,t}) > \beta \right) \leq \eta.
	\end{equation*}
\end{restatable}

In principle, we would design \emph{no-regret} algorithms, i.e., whose pseudo-regret $R_T(\mathfrak{U})$ increases sub-linearly w.r.t.~$T$, and, at the same time, that are $\eta$-safe, with $\eta$ sublinearly increasing in (or is independent of) $T$.
However, the following theorem shows that no online learning algorithm can provide a sublinear pseudo-regret while guaranteeing safety.
\begin{restatable}[Pseudo-regret/safety tradeoff]{thm}{thmtrade}. 
	For every $\epsilon > 0$ and time horizon $T$, there is no algorithm with pseudo-regret smaller than $(1/2-\epsilon)\,T$ and that violates (in expectation) the constraints less than $(1/2-\epsilon)\,T$ times.
\end{restatable}
This impossibility result is crucial in practice, showing that no online learning algorithm can theoretically guarantee both a sublinear regret and a sublinear number of violations of the constraints.
Therefore, advertisers must accept a tradeoff between the two requirements in real-world applications.

\section{Proposed Methods} \label{sec:4}

The above-defined problem is closely related to the Combinatorial MAB formulation provided by~\citet{chen2013combinatorial}, but designing algorithms for this setting requires dealing with an additional challenge. More specifically, in the classical Combinatorial MAB framework, the set of feasible super-arms is chosen from a fixed and known set. Instead, in the framework we defined, the set of constraints influences the set of arms that are feasible according to the constraints, which in turn changes the set of feasible super-arms. In the following, we propose two algorithms that carefully define the feasible set of arms at each round to deal with this issue.

In this work, we assume that the number of clicks $n_j(\cdot)$ and costs $c_j(\cdot)$ functions are Gaussian Processes (GPs)~\cite{rasmussen2006gaussian}. This modeling approach allows to capture the regularity of the considered functions without constraining them into a specific family. Such a mild assumption has already been used successfully in the ad optimization literature, e.g., in the work by~\citet{nuara2018combinatorial}. From now on, and for the sake of presentation, we will assume that this assumption holds, however, we remark that using the techniques presented in the work by~\citet{srinivas2010gaussian}, one can extend the methodology we present to an even larger class of functions.

We provide the pseudo-code of our algorithms, namely \textsf{GCB} and \alg, in Algorithm~\ref{alg:highlevel}, which solves the problem in Equations~(\ref{formulation:objectivefunction})-(\ref{formulation:budgetconstraint}) online while guaranteeing sublinear regret or $\eta$-safety, respectively.
Algorithm~\ref{alg:highlevel} is divided into an \emph{estimation phase} (Lines~\ref{line:estini}-\ref{line:estend}) based on GPs to model the parameters whose values are unknown, and an \emph{optimization subroutine} to solve the optimization problem once given the estimates (Line~\ref{line:opt}).
Finally, in the last phase, the newly acquired data are used to improve the GP models that will be used in the following round (Lines~\ref{line:updini}-\ref{line:updend}).

We remark that \citet{accabi2018gaussian} propose the \textsf{GCB} algorithm to face general combinatorial bandit problems where the arms are partitioned in subsets and the payoffs of the arms belonging to the same subset are modeled with a GP.
To obtain theoretical sublinear guarantees on the regret for our online learning problem, we use a specific definition of $\bm{\mu}$ vector, making Algorithm~\ref{alg:highlevel} be an extension of \textsf{GCB} when the payoffs and constraints are functions whose parameters are modeled by multiple independent GPs.
With a slight abuse of terminology, we refer to this extension as \textsf{GCB}.

\begin{algorithm}[t!]
	\caption{\textsf{GCB} and \alg{} pseudo-code}
	\label{alg:highlevel}
	\begin{algorithmic}[1]
		\begin{small}
			\Statex \textbf{Input}: sets of bid values $X_1, \ldots, X_N$, ROI threshold $\lambda$, daily budget $\beta$
			\State Initialize the GPs for the number of clicks and costs
			\For{$t \in \{1, \ldots, T \}$}
			\For{$j \in \{1, \ldots, N \}$} \label{line:estini}
			\For{$x \in X_j$}
			\State Compute $\hat{n}_{j,t-1}(x)$ and $\hat{\sigma}_{j,t-1}^n(x)$ according to Eq.~\eqref{eq:ncmean} and~\eqref{eq:ncsigma}, respectively
			\State Compute $\hat{c}_{j,t-1}(x)$ and $\hat{\sigma}_{j,t-1}^c(x)$ according to Eq.~\eqref{eq:comean} and~\eqref{eq:cosigma}, respectively
			\EndFor
			\EndFor
			\State Compute $\boldsymbol{\mu}$ according to Eq.~\eqref{eq:gcb} for \textsf{GCB} or~\eqref{eq:gcbsafe} for \alg{} \label{line:estend}
			\State Call \opt{} to get a solution $\left\{ \realbi_{j,t} \right\}_{j=1}^N$ \label{line:opt}
			\State Set the prescribed allocation $\left\{ \realbi_{j,t} \right\}_{j=1}^N$ during round $t$
			\State Get revenue $\sum_{j=1}^N v_j \ \tilde{n}_j(\hat{\bi}_{j,t})$ and pay costs $\sum_{j=1}^N \tilde{c}_j(\hat{\bi}_{j,t})$\label{line:updini}
			\State Update the GPs using the new information $\tilde{n}_{j,t}(\hat{x}_{j,t})$ and $\tilde{c}_{j,t}(\hat{x}_{j,t})$ \label{line:updend}
			\EndFor
		\end{small}
	\end{algorithmic}
	\vspace{-0.3cm}
\end{algorithm}

\subsection{Estimation Phase}

In Algorithm~\ref{alg:highlevel}, GPs are used to model functions of the number of clicks $n_j(\cdot)$ and the costs $c_j(\cdot)$.
GPs use the noisy realization of the actual number of clicks $\tilde{n}_{j,h}(\hat{x}_{j,h})$ collected from each subcampaign $C_j$ for every previous round $h \in \{1, \ldots, t-1\}$ to generate, for every bid $x \in X_j$, the estimates for the expected value $\hat{n}_{j,t-1}(x)$ and the standard deviation of the number of clicks $\hat{\sigma}_{j,t-1}^n(x)$.
Analogously, using the noisy realizations of the actual cost $\tilde{c}_{j,h}(\hat{x}_{j,h})$, with $h \in \{1, \ldots, t-1\}$, GPs generate, for every bid $x \in X_j$, the estimates for the expected value $\hat{c}_{j,t-1}(x)$ and the standard deviation of the cost $\hat{\sigma}_{j,t-1}^c(x)$.
Formally, the above values are computed as follows:
\begin{align}
	& \hat{n}_{j,t-1}(\bi) \hspace{-0.1cm} := \bm{k}_{j,t-1}(\bi)^\top [K_{j,t-1} + \sigma_n^2 I]^{-1} \bm{k}_{j,t-1}(\bi), \label{eq:ncmean} \\
	& \hat{\sigma}_{j,t-1}^n(\bi) := k_j(\bi, \bi) \nonumber \\
	&\qquad \qquad - \bm{k}_{j,t-1}^\top (\bi) [K_{j,t-1} + \sigma_n^2 I]^{-1} \bm{k}_{j,t-1}(\bi), \label{eq:ncsigma} \\
	& \hat{c}_{j,t-1}(\bi) := \bm{h}_{j,t-1}^\top (\bi) [H_{j,t-1} + \sigma_c^2 I]^{-1} \bm{h}_{j,t-1}(\bi), \label{eq:comean} \\
	& \hat{\sigma}_{j,t-1}^c(x) := h_j(\bi, \bi) \nonumber \\
	&\qquad \qquad - \bm{h}_{j,t-1}^\top (\bi) [H_{j,t-1} + \sigma_c^2 I]^{-1} \bm{h}_{j,t-1}(\bi), \label{eq:cosigma}
\end{align}
where $k_j(\cdot, \cdot)$ and $h_j(\cdot, \cdot)$ are the kernels for the GPs over the number of clicks and costs, respectively, $K_{j,t-1}$ and $H_{j,t-1}$ are the Gram matrix over the bids selected during the rounds $\{1, \ldots t-1\}$ for the two GPs, $\sigma_n^2$ and $\sigma_c^2$ are the variance of the noise of the GPs, $\bm{k}_{j,t-1}(\bi)$ and $\bm{h}_{j,t-1}(\bi)$ are vectors built computing the kernel between the training bids $\{\bi_{j,h}\}_{h=1}^{t-1}$and the current bid $x$, and $I$ is the identity matrix of order $t-1$.
For further details on using GPs, we point an interested reader to the work by~\citet{rasmussen2006gaussian}.

The estimation subroutine returns the vector $\boldsymbol{\mu}$ of parameters characterizing the specific instance of the optimization problem.
More specifically, it is composed as follows:
\begin{align*}
	&\boldsymbol{\mu} := (\bar{\mathbf{w}}_1, \ldots, \bar{\mathbf{w}}_N, \underline{\mathbf{w}}_1, \ldots, \underline{\mathbf{w}}_N, -\bar{\mathbf{c}}_1, \ldots, -\bar{\mathbf{c}}_N),\\
	&\bar{\mathbf{w}}_j := (\bar{w}_j(x_1), \ldots, \bar{w}_j(x_{|X_j|})),\\
	&\underline{\mathbf{w}}_j := (\underline{w}_j(x_1), \ldots, \underline{w}_j(x_{|X_j|})),\\
	&\bar{\mathbf{c}}_j := (\bar{c}_j(x_1), \ldots, \bar{c}_j(x_{|X_j|})),
\end{align*}
where $\overline{w}_j(x_j) := v_j\, \overline{n}_j (x_j)$ and $\underline{w}_j(x_j) := v_j\, \underline{n}_j (x_j)$ denote different estimates for the revenue of a subcampaign $C_j$, and $\bar{\mathbf{c}}_j$ characterizes the costs for subcampaign $C_j$.
In the optimization subroutine, $\overline{w}_j(x_j)$ and $\underline{w}_j(x_j)$ will be used to compute the value of Equation~(\ref{formulation:objectivefunction}) and Equation~(\ref{formulation:roiconstraint}), respectively, and $\overline{c}_j(x_j)$ to compute the value of Equations~(\ref{formulation:roiconstraint})-(\ref{formulation:budgetconstraint}).
We use $\overline{h}$ and $\underline{h}$ to denote potentially different estimated values of a generic function $h$ used by the learning algorithms in the next sections.
The proposed algorithms are derived from two different procedures to compute $\boldsymbol{\mu}$.
Formally, we have that the elements of $\boldsymbol{\mu}$ are defined as follows for \textsf{GCB}:
\begin{align}
	& \overline{w}_j(\bi) \hspace{-2pt} = \hspace{-2pt} \underline{w}_j(\bi) \hspace{-2pt} := v_j \hspace{-2pt} \left[ \hat{n}_{j,t-1}(\bi) + \sqrt{b_{t-1}} \hat{\sigma}_{j,t-1}^n(\bi) \right], \\ 
	& \overline{c}_j(\bi) := \hat{c}_{j,t-1}(\bi) - \sqrt{b_{t-1}} \hat{\sigma}_{j,t-1}^c(\bi), \label{eq:gcb}
\end{align}
and for \alg{}:
\begin{align}
	& \overline{w}_j(\bi) := v_j \left[ \hat{n}_{j,t-1}(\bi) + \sqrt{b_{t-1}} \hat{\sigma}_{j,t-1}^n(\bi) \right],\\
	& \underline{w}_j(\bi) := v_j \left[ \hat{n}_{j,t-1}(\bi) - \sqrt{b_{t-1}} \hat{\sigma}_{j,t-1}^n(\bi) \right],\\
	&\overline{c}_j(\bi) := \hat{c}_{j,t-1}(\bi) + \sqrt{b_{t-1}} \hat{\sigma}_{j,t-1}^c(\bi), \label{eq:gcbsafe}
\end{align}
where $b_t$ is an uncertainty term that is appropriately set in Section~\ref{sec:5}.

The \textsf{GCB} and \alg{} algorithm relies on the idea that the elements in the $\bm{\mu}$ vector represent the statistical upper/lower bounds to the expected values of the number of clicks and costs.
This follows a common choice in the bandit literature to incentivize exploration for uncertain quantities (a.k.a.~optimism in the face of uncertainty principle).

\begin{algorithm}[t!]
	\caption{\opt{} subroutine}
	\begin{small}
		\label{alg:bbo3}
		\begin{algorithmic}[1]
			\Statex \textbf{Input}: sets of bid values $X_1, \ldots, X_N$, set of cumulative cost values $Y$, set of revenue values $R$, vector $\boldsymbol{\mu}$, ROI threshold $\lambda$
			\State Initialize $\textbf{x}^{y,r} = [\;\;], \ \forall \ y \in Y, \forall \ r \in R$ \label{line:init2}
			\For{ $j \in \{1, \ldots, N\}$ }
			\For{ $y \in Y$ }
			\For{ $r \in R$ }
			\State Update $S(y, r)$ according to Equation~\eqref{eq:setS} \label{line:safej}
			\State $\textbf{x}_{\textnormal{next}}^{y, r} = \arg \max_{\textbf{s} \in S(y,r)} \sum_{i=1}^j \overline{w}_i(s_i)$ \label{line:argmax}
			\State $M(y, r) = \max_{\textbf{s} \in S(y,r)} \sum_{i=1}^j \overline{w}_i(s_i)$ \label{line:maxim}
			\EndFor
			\EndFor
			\State $\textbf{x}^{y,r} = \textbf{x}_{\textnormal{next}}^{y,r}, \ \forall \ y \in Y, \forall \ r \in R$
			\EndFor
			\State $(y^*,r^*) = \arg \max_{(y,r) \in Y \times R, \ \frac{r}{y} \geq \lambda} M(y,r)$ \label{line:last}
			\State \textbf{Output:} $\textbf{x}^{y^*,r^*}$ \label{line:return}
		\end{algorithmic}
	\end{small}
	\vspace{-0.1cm}
\end{algorithm}

\subsection{Optimization Subroutine}

The pseudo-code of the \opt{} subroutine, solving the problem in Equations~(\ref{formulation:objectivefunction})-(\ref{formulation:budgetconstraint}) with a dynamic programming approach, is provided in Algorithm~\ref{alg:bbo3}.
It takes as input the set of the possible bid values $X_j$ for each subcampaign $C_j$, the set of the possible cumulative cost values $Y$ such that $\max_{y \in Y} y = \beta$, the set of the possible revenue values $R$, an ROI threshold $\lambda$, and a vector $\boldsymbol{\mu}$ characterizing the optimization problem.
In particular, if the functions are known beforehand, it holds $\overline{h} = \underline{h} = h$ for both $h = w_j$ and $h = c_j$.
For the sake of clarity, $\overline{w}_j(x)$ is used in the objective function, while $\underline{w}_j(x)$ and $\overline{c}_j(x)$ are used in the constraints.
At first, the subroutine initializes the vectors $\textbf{x}^{y,r} = [\;\;],$ $\forall \ y \in Y, \forall \ r \in R$  that will contain the (partial) bid allocation providing the largest revenue among the ones having a cost of at most $y$ and total revenue of at least $r$ (Line~\ref{line:init2}).
At each iteration, for every pair $(y, r)$, the subroutine stores in $\textbf{x}^{y,r}$ the optimal set of bids for subcampaigns $C_1, \ldots, C_j$ that maximizes the objective function and stores the corresponding optimum value in $M(y, r)$.
More specifically, at the $j$-th iteration, the computation of the optimal bids is performed by evaluating at first a set of candidate solutions $S(y, r)$:
\begin{align}
	& S(y, r) := \bigcup_{x \in X_j, y' \in y, r' \in R} \Big\{ \textbf{s} = [\textbf{x}^{y',r'}, x] \; | \; y' + \overline{c}_j(x) \leq y \nonumber\\
	&\wedge r' + \underline{w}_j(x) \geq r \Big\}. \label{eq:setS}
\end{align}
Intuitively, this set is built by combining the optimal (partial) bid allocation $\textbf{x}^{y',r'}$ computed at the $(j-1)$-th iteration with one of the bids $x \in X_j$ available for the $j$-th subcampaign.
After that, the subroutine selects $\textbf{x}_{\textnormal{next}}^{y, r}$ as the element in $S(y, r)$ that maximizes the revenue (Line~\ref{line:argmax}) and updates the corresponding revenue in $M(y, r)$ (Line~\ref{line:maxim}).
Once the subroutine finished iterating over the campaigns $C_j$, it computes the optimal cost/revenue pair $(y^*, r^*)$ satisfying the ROI constraint (Line~\ref{line:last}) and return the corresponding bid allocation $\textbf{x}^{y^*, r^*}$ (Line~\ref{line:return}).\footnote{
	Notice that we do not need to check the budget constraint since, by construction, all the bid allocations $\textbf{x}_{\textnormal{next}}^{y, r}$ have $y \in Y$, and, therefore, satisfy the budget constraint.}
The following property states the optimality of the proposed subroutine w.r.t.~the considered optimization problem:
\begin{restatable}[Optimality]{thm}{optimality}
	The \opt{} subroutine returns the optimal solution to the problem in Equations~(\ref{formulation:objectivefunction})-(\ref{formulation:budgetconstraint}) when $\overline{w}_j(x) = \underline{w}_j(x) = v_j\,n_j(x)$ and $\overline{c}_j(x) = c_j(x)$ for each $j \in \{1, \ldots, N\}$ and the values of revenues and costs are in $R$ and $Y$, respectively.
\end{restatable}

As a final remark, we mention that the running time of \opt{} is quadratic in $|Y|$ and $|R|$ and linear in the number of available bids $|X_j|$ and number of campaign $N$.
See the supplementary material for details.

\section{Theoretical Guarantees}
\label{sec:5}
Moreover, we will show how to have sublinear guarantee on both of them by allowing small violations of the constraints.
Let us first define the \emph{maximum information gain} $\gamma_{j,t}$ of the GP modeling the number of clicks of subcampaign $C_j$ at round $t$, formally defined as:
\begin{equation*}
	\gamma_{j,t} := \frac{1}{2} \max_{(x_{j,1}, \ldots, x_{j,t}) \in X_j^t} \left|I_t + \frac{\Phi(x_{j,1}, \ldots, x_{j,t})}{\sigma_n^2} \right|,
\end{equation*}
where $I_t$ is the identity matrix of order $t$, $\Phi(x_{j,1}, \ldots, x_{j,t})$ is the Gram matrix of the GP computed on the vector $(x_{j,1}, \ldots, x_{j,t})$, and $\sigma \in \mathbb{R}^+$ is the noise standard deviation.

\subsection{Guaranteeing Sublinear Pseudo-regret: \textsf{GCB}} \label{sec:5.1}

The \textsf{GCB} algorithm, which is based on the optimist in the face of uncertainty principle, provides the following pseudo-regret bound:
\begin{restatable}[\textsf{GCB} pseudo-regret]{thm}{gcbregret} \label{thm:gcbregret}
	Given $\delta \in (0, \ 1)$, \textsf{GCB} applied to the problem in Equations~(\ref{formulation:objectivefunction})-(\ref{formulation:budgetconstraint}), with probability at least $1 - \delta$, suffers from a pseudo-regret of:
	\begin{equation*}
		R_T(\textsf{GCB }) \leq \sqrt{\frac{8 T v^2_{\max} N^3 b_T}{\ln(1 + \sigma_n^2)} \sum_{j=1}^N \gamma_{j,T}}~~,
	\end{equation*}
	where $b_t := 2 \ln \left( \frac{\pi^2 N Q T t^2 }{3 \delta} \right)$ is an uncertainty term used to guarantee the confidence level required by \textsf{GCB}, $v_{\max} := \max_{j \in \{1, \ldots, N\}} v_j$ is the maximum value per click over all subcampaigns, and $Q := \max_{j \in \{1, \ldots, N\}} |X_j|$ is the number of bids in a subcampaign.
\end{restatable}

We remark that the upper bound provided in the above theorem is expressed in terms of the maximum information gain $\gamma_{j,T}$ of the GPs over the number of clicks.
The problem of bounding $\gamma_{j,T}$ for a generic GP has already been addressed by~\cite{srinivas2010gaussian}, where the authors present the bounds for the squared exponential kernel $\gamma_{j,T} = \mathcal{O}((\ln{T})^{2})$ for $1$-dimensional GPs.
Notice that, thanks to the previous result, the \textsf{GCB} algorithm using squared exponential kernels suffers from a sublinear pseudo-regret since the terms $\gamma_{j,T}$ is bounded by $\mathcal{O}( (\ln{T})^{2})$, and the bound in Theorem~\ref{thm:gcbregret} becomes $\mathcal{O}(N^{3/2} (\ln{T})^{5/2}) \sqrt{T})$.
However, the \textsf{GCB} algorithm violates (in expectation) the constraints a linear number of times in $T$.
\begin{restatable}[\textsf{GCB} safety]{thm}{gcbconst} \label{thm:gcbconst}
	Given $\delta \in (0, \ 1)$, \textsf{GCB} applied to the problem in Equations~(\ref{formulation:objectivefunction})-(\ref{formulation:budgetconstraint}) with $b_t := 2 \ln \left( \frac{\pi^2 N Q T t^2 }{3 \delta} \right)$ is $\eta$-safe where $\eta \geq T - \frac{\delta}{2 N Q T}$ and, therefore, the number of constraints violations is linear in $T$.
\end{restatable}

This result states that if we apply the \textsf{GCB} algorithm, we expect to have a large revenue over the time horizon $T$ at the cost of violating the ROI and/or the budget constraints most of the time over the learning period.
Therefore, in practical cases, such an algorithm might perform poorly regarding ROI over the entire time horizon $T$.
As highlighted before, this behaviour might lead to a premature stop of the algorithm from the business unit.

\subsection{Guaranteeing Safety: \alg} 
\label{sec:GCBsafe}

The \alg{} algorithm uses different bounds than \textsf{GCB} to evaluate the constraints and have stronger guarantees about their satisfaction.
While the estimates for the revenue of the algorithm (Equation~(\ref{formulation:objectivefunction})) are estimated using upper bounds, for the constraints (Equation~(\ref{formulation:roiconstraint})-(\ref{formulation:budgetconstraint})), we used statistical lower bounds to guarantee they are satisfied at every round with high probability.
This choice comes at the cost of a linear worst-case performance in terms of pseudo-regret:
\begin{restatable}[\alg{} pseudo-regret]{thm}{gcbsaferegret} \label{thm:gcbsaferegret}
	Given $\delta \in (0, \ 1)$, \alg{} with $b_t := 2 \ln \left( \frac{\pi^2 N Q T t^2 }{3 \delta} \right)$, applied to the problem in Equations~(\ref{formulation:objectivefunction})-(\ref{formulation:budgetconstraint}) suffers from a pseudo-regret $R_t(\alg) = \Theta(T)$.
\end{restatable}
However, \alg{} violates the ROI and budget constraints only a constant number of times w.r.t.~$T$.
\begin{restatable}[\alg{} safety]{thm}{gcbsafeconst} \label{thm:gcbsafeconst}
	Given $\delta \in (0, \ 1)$, \alg{} with $b_t := 2 \ln \left( \frac{\pi^2 N Q T t^2 }{3 \delta} \right)$, applied to the problem in Equations~(\ref{formulation:objectivefunction})-(\ref{formulation:budgetconstraint}) is $\delta$-safe and the number of constraints violations is constant in $T$.
\end{restatable}

This result also provides results on the total amount of violation of the two constraints, i.e., $\sum_{t=1}^T \sum_{j=1}^N \co_j(\bi_{j,t}) - T \ \beta$ and $T \lambda - \sum_{t=1}^T \frac{\sum_{j=1}^N v_j n_j(\bi_{j,t})}{\sum_{j=1}^N \co_j(\bi_{j,t})}$, we expect from \alg{}. Indeed, it states that with probability at least $1 - \delta$, \alg{} provides a null violation of the constraints. Conversely, for \textsf{GCB}, one can only guarantee a sublinear total amount of violation for each constraint. See the supplementary material for details.

In what follows, we show that if we accept tolerance in the violation of the constraints, an adaptation of \alg{} can be exploited to obtain also a sublinear pseudo-regret.

\subsection{Guaranteeing Sublinear Pseudo-regret and Safety with Tolerance: \alg$(\psi,\phi)$}
\label{sec:GCBrelax}

Given an instance of the problem in Equations~(\ref{formulation:objectivefunction})-(\ref{formulation:budgetconstraint}) that we call \emph{original problem}, we build an \emph{auxiliary problem} in which we slightly relax the ROI and budget constraints.
Formally, the \alg$(\psi, \phi)$ is the \alg{} applied to the auxiliary problem in which the parameters $\lambda$ and $\beta$ have been substituted with $\lambda - \psi$ and $\beta + \phi$, with $\psi \in (0, \lambda)$ and $\psi \in \mathbb{R}^+$, formally:
\begin{subequations}
	\begin{align}
		\max_{(\bi_{1,t}, \ldots, \bi_{N,t}) \in X_1 \times \ldots \times X_N} & \sum_{j=1}^N v_j \ n_j(\bi_{j,t}) \label{formulation:auxinit}\\
		\text{s.t.} \qquad \qquad &\frac{\sum_{j=1}^N v_j \ n_j(\bi_{j,t}) }{\sum_{j=1}^N \ \co_j(\bi_{j,t}) } \geq \lambda - \psi, \\
		& \sum_{j=1}^{N} c_j(x_{j,t}) \leq \beta + \phi. \label{formulation:auxend}
	\end{align}
\end{subequations}

The regret of this algorithm w.r.t.~the original problem is:

\begin{restatable}[\alg{}$(\psi, \phi)$ pseudo-regret]{thm}{gcbsaferelaxboth} \label{thm:gcbsaferelaxboth}
	Given $\delta' \leq \delta$ and setting
	\begin{equation*}
		\psi = 2 \frac{\beta_{opt}+n_{\max}}{\beta^2_{opt}} \sigma_c \sum_{j=1}^N v_j \sqrt{2 \ln\left(\frac{\pi^2 N Q T^3}{3 \delta'}\right)}
	\end{equation*}
	and
	\begin{equation*}
		\qquad \phi = 2 N \sigma_c \sqrt{2 \ln\left(\frac{\pi^2 N Q T^3}{3 \delta'}\right)},
	\end{equation*}
	\alg{}$(\psi, \phi)$ provides a pseudo-regret w.r.t.~the optimal solution to the original problem of $\mathcal{O} \left( \sqrt{T \sum_{j=1}^N \gamma_{j,T}} \right)$ with probability at least $1 - \delta - \frac{\delta'}{Q T^2}$.
\end{restatable}

As a direct corollary of Theorem~\ref{thm:gcbsafeconst}, \alg$(\psi, \phi)$, w.h.p., does not violate the ROI constraint of the original problem by more than the tolerance $\psi$ and the budget constraint of the original problem by more than the tolerance $\phi$, and is $\delta$-safe w.r.t.~the auxiliary problem:

\begin{restatable}[\alg{}$(\psi, \phi)$ safety]{cor}{gcbsaferelaxbothsafe} \label{thm:gcbsaferelaxbothsafe}
	Given $\delta \in (0, \ 1)$, \alg$(\psi, \phi)$ with $b_t := 2 \ln \left( \frac{\pi^2 N Q T t^2 }{3 \delta} \right)$ is $\delta$-safe w.r.t.~the constraints of the auxiliary problem in Equations~(\ref{formulation:auxinit})-(\ref{formulation:auxend}).
\end{restatable}

Some comments are in order. The above results state that if we allow a violation of at most $\psi$ of the ROI constraint and of $\phi$ of the budget one, the result provided in Theorem~\ref{thm:inapprox} can be circumvented.

Notice that the magnitude of the violation $\psi$ required in Theorem~\ref{thm:gcbsaferelaxboth} increases linearly in the maximum number of clicks $n_{max}$ and $\sum_{j=1}^N v_j$, that, in turn, increases linearly in the number of sub-campaigns $N$.
This suggests that in large instances, this value may be large.
However, in practice, the maximum number of clicks of a sub-campaign $n_{\max}$ is a sublinear function in the optimal budget $\beta_{opt}$, and usually, it goes to a constant as the budget spent goes to infinity.
Moreover, the number of sub-campaigns $N$ usually depends on the budget, \emph{i.e.}, the budget planned by the business units is linear in the number of sub-campaigns.
As a result, $\beta_{opt}$ is of the same order of $\sum_{j=1}^N v_j$, and therefore, since $n_{max}$ is sublinear in $\beta_{opt}$ and $\sum_{j=1}^N v_j$ is of the order of $\beta_{opt}$, the final expression of $\psi$ is sub-linear in $\beta_{opt}$. This means that the value of $\psi$ to satisfy the assumption needed by Theorem~\ref{thm:gcbsaferelaxboth} goes to zero as $\beta_{opt}$ increases.

\begin{figure*}[th!]
	\centering
	\captionsetup[subfigure]{oneside,margin={1.05cm,0cm}}
	\subfloat[]{\scalebox{0.67}{\includegraphics{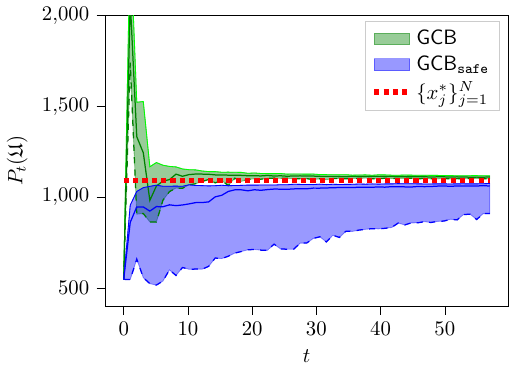}\label{fig:exp1_rev}}}
	\subfloat[]{\scalebox{0.67}{\includegraphics{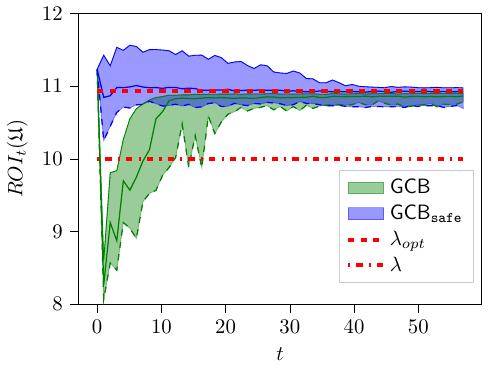}\label{fig:exp1_roi}}}
	\subfloat[]{\scalebox{0.67}{\includegraphics{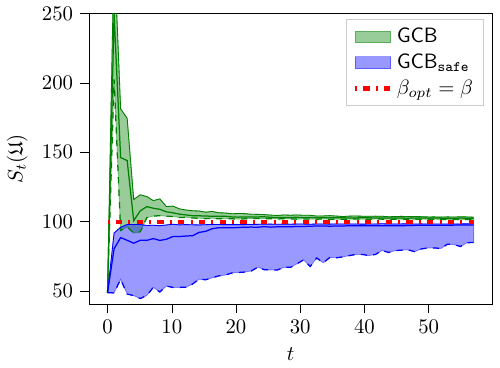}\label{fig:exp1_costs}}}	
	\vspace{-0.4cm}
	\caption{Results of Experiment \#1: Daily revenue (a), ROI (b), and spend (c) obtained by \textsf{GCB} and \alg{}. Dashed lines correspond to optimal values for revenue and ROI, while dash-dotted lines correspond to values of ROI and budget constraints.}
	\label{fig:exp1}
\end{figure*}

Conversely, the most relevant dependence on the magnitude of $\phi$ is the number of campaigns $N$.
This is reasonable since the more subcampaigns there are, the more potential variance we have over the costs, which should be balanced with a larger violation of the constraint.
This suggests that the \alg{}$(\psi, \phi)$ will not be effective for large instances of the analysed optimization problem.

\section{Experimental Evaluation}
\label{sec:6}
We experimentally evaluate our algorithms in terms of pseudo-regret and safety in synthetic settings.
The adoption of synthetic settings allows us to evaluate our algorithms in realistic scenarios and, at the same time, to have an optimal clairvoyant solution necessary to measure the algorithms' pseudo-regret and safety.
In the following experiment, we show that \textsf{GCB} suffers from significant violations of both ROI and budget constraints even in simple settings, while \alg{} does not.\footnote{
	Additional experiments and details useful for the complete reproducibility of our results are provided in the supplementary material. Code available at \url{https://github.com/oi-tech/safe_bid_opt}.}

\subsection{Experiment \#1: evaluating constraint violation with \textsf{GCB} and \alg{}}\label{sec:6.1}

We simulate $N = 5$ subcampaigns, with $|X_j| = 201$ bid values evenly spaced in $[0, \ 2]$, $|Y| = 101$ cost values evenly spaced in $[0, \ 100]$, and $|R| = 151$ revenue values evenly spaced in $[0, \ 1200]$.
For a generic subcampaign $C_j$, at every $t$, the daily  number of clicks is returned by the function $\tilde{n}_j(x) := \theta_j(1 - e^{-x/\delta_j}) + \xi^n_j$ and the daily cost by the function $\tilde{c}_j(x) = \alpha_j(1 - e^{-x/\gamma_j}) + \xi^c_j$, where $\theta_j \in \mathbb{R}^+$ and $\alpha_j \in \mathbb{R}^+$ represent the maximum achievable number of clicks and cost for subcampaign $C_j$ in a single day, $\delta_j \in \mathbb{R}^+$ and $\gamma_j \in \mathbb{R}^+$ characterize how fast the two functions reach a saturation point, and $\xi^n_j $ and $\xi^c_j$ are noise terms drawn from a $\mathcal{N}(0,1)$ Gaussian distribution 
(these functions are customarily used in the advertising literature, \emph{e.g.,} by~\citet{kong2018combinatorial}).
We assume a unitary value for each click, \emph{i.e.,} $v_j = 1$ for each $j \in \{1, \ldots, N \}$.
The values of the parameters of cost and revenue functions of the subcampaigns are specified in the supplementary material.
We set a daily budget $\beta = 100$, $\lambda = 10$ in the ROI constraint, and a time horizon $T = 60$.
Notice that in this setting, at the optimal solution, the budget constraint is active, while the ROI constraint is not. 
For both \textsf{GCB} and \alg{}, the kernels for the number of clicks GPs $k(x, x')$ and for the costs GPs $h_j(x, x')$ are squared exponential kernels of the form $\sigma^2_f \,\exp \left\{- \frac{(x - x')^2}{l} \right\}$ for every $x, x' \in X_j$, where the parameters $\sigma_f \in \mathbb{R}^+$ and $l \in \mathbb{R}^+$ are estimated from data, as suggested by~\citet{rasmussen2006gaussian}.
The confidence for the algorithms is $\delta = 0.2$.
We evaluate the algorithms in terms of daily revenue $P_t(\mathfrak{U}) := \sum_{j=1}^N v_j n_j(\realbi_{j,t})$, daily ROI: $ROI_t(\mathfrak{U}) := \frac{\sum_{j=1}^N v_j \ n_j(\realbi_{j,t}) }{\sum_{j=1}^N \ \co_j(\realbi_{j,t}) }$, and  daily spend: $S_t(\mathfrak{U}) : = \sum_{j=1}^{N} c_j(\realbi_{j,t})$.
We perform $100$ independent runs for each algorithm. 

\paragraph{Results}
In Figure~\ref{fig:exp1}, for the daily revenue, ROI, and spend achieved by \textsf{GCB} and \alg{} at every $t$, we show the $50_{th}$ percentile with solid lines and the $90_{th}$ and $10_{th}$ percentiles with dashed lines surrounding the semi-transparent area.
While \textsf{GCB} achieves a larger revenue than \alg{}, it violates the budget constraint over the entire time horizon and the ROI constraint in the first $7$ days in more than $50\%$ of the runs.
This happens because, in optimal solutions, the ROI constraint is not active, while the budget constraint is. 
Conversely, \alg{} satisfies the budget and ROI constraints over the time horizon for more than $90\%$ of the runs and has a slower convergence to the optimal revenue.
If we focus on the median revenue, \alg{} has a similar behaviour to \textsf{GCB} for $t > 15$. This makes \alg{} a good choice, even in terms of overall revenue.
However, it is worth noticing that, in the $10\%$ of the runs, \alg{} does not converge to the optimal solution before the end of the learning period.
These results confirm our theoretical analysis that limiting the exploration to safe regions leads the algorithm to a large regret.
%

\begin{figure*}[t!]
	\centering
	\captionsetup[subfigure]{oneside,margin={1.05cm,0cm}}\textbf{}
	\subfloat[]{\scalebox{0.67}{\includegraphics{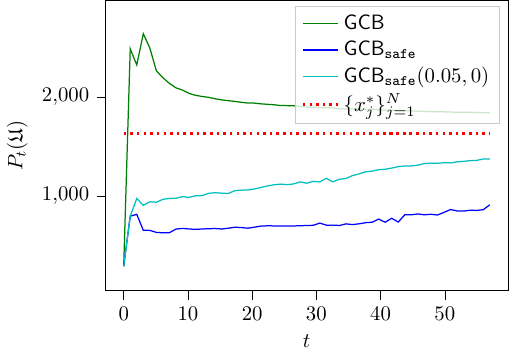}\label{fig:exp3_rev}}}
	\subfloat[]{\scalebox{0.67}{\includegraphics{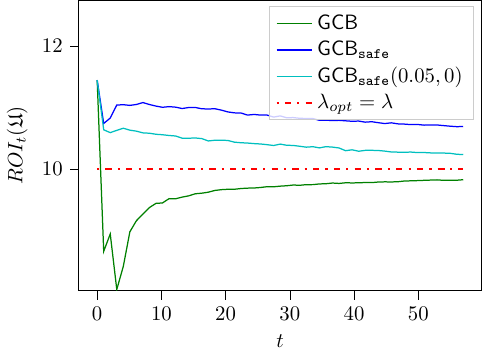}\label{fig:exp3_roi}}}
	\subfloat[]{\scalebox{0.67}{\includegraphics{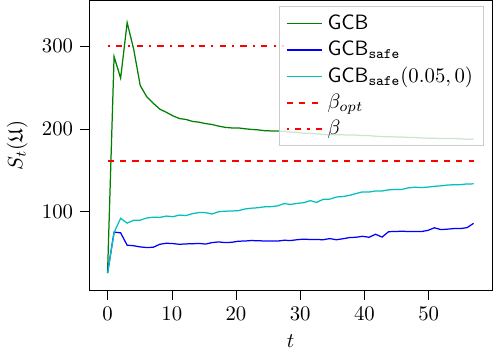}\label{fig:exp3_costs}}}
	\vspace{-0.4cm}
	\caption{Results of Experiment \#2: Median values of the daily revenue (a), ROI (b) and spend (c) of \textsf{GCB}, \alg{}, and \alg{}($0.05,0$).}
	\label{fig:exp2}
\end{figure*}

\begin{figure*}[t!]
	\centering
	\captionsetup[subfigure]{oneside,margin={1.05cm,0cm}}
	\subfloat[]{\scalebox{0.67}{\includegraphics{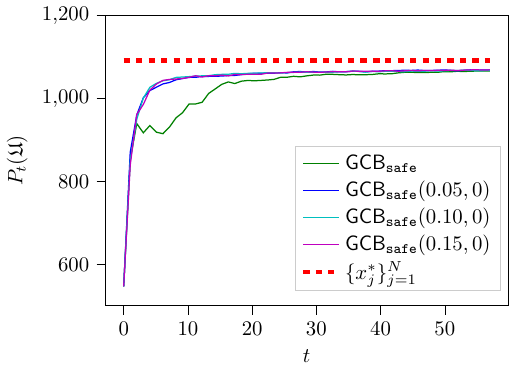}\label{fig:exp2_rev}}}
	\subfloat[]{\scalebox{0.67}{\includegraphics{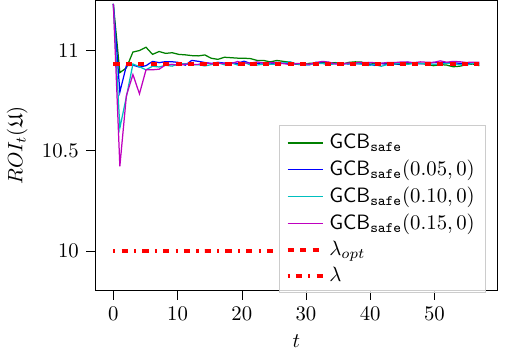}\label{fig:exp2_roi}}}
	\subfloat[]{\scalebox{0.67}{\includegraphics{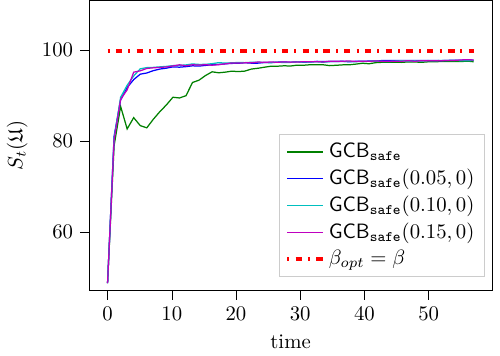}\label{fig:exp2_costs}}}
	\vspace{-0.4cm}
	\caption{Results of Experiment \#3: Median values of the daily revenue (a), ROI (b) and spend (c) obtained by \alg{}$(\psi,0)$ with different values of $\psi$.}
	\label{fig:exp3}
\end{figure*}

\subsection{Experiment \#2: \textsf{GCB}, \alg, and \alg$(\psi, 0)$ if the ROI constraint is active}\label{sec:6.3}

We study a setting where the ROI constraint is active at the optimal solution, \emph{i.e.}, $\lambda = \lambda_{opt}$, while the budget constraint is not. 
This means that, at the optimal solution, the advertiser would have an extra budget to spend.
However, if such a budget is not spent, the ROI constraint would be violated otherwise. 
The experimental setting is the same as Experiment~\#1, except that we set the budget constraint as $\beta = 300$.
The optimal daily spend is $\beta_{opt} = 161$.

\paragraph{Results}
In Figure~\ref{fig:exp3}, we show the median values of the daily revenue, the ROI, and the spend of \textsf{GCB}, \alg{}, \alg{}$(0.05, 0)$ obtained with $100$ independent runs.
The $10\%$ and $90\%$ of the quantities provided by \textsf{GCB}, \alg{}, and \alg{}$(0.05, 0)$ are reported in the supplementary material.
We notice that, even in this setting, \textsf{GCB} violates the ROI constraint for the entire time horizon and the budget constraint in $t = 6$ and $t = 7$.
However, it achieves a revenue larger than that of the optimal constrained solution. 
On the other side, \alg{} always satisfies both constraints, but it does not perform enough exploration to quickly converge to the optimal solution.
We observe that it is sufficient to allow a tolerance in the ROI constraint violation by slightly perturbing the input value $\lambda$ ($\psi = 0.05$, corresponding to a violation of the constraint by at most $0.5\%$) to make \alg$(\psi,\phi)$ capable of approaching the optimal solution while satisfying both constraints for every $t \in \{ 0, \ldots, T \}$.
This suggests that, in real-world applications, the use of \alg$(\psi,\phi)$ with a small tolerance represents an effective solution, providing guarantees on the violation of the constraints while returning high revenue values.

\subsection{Experiment \#3: evaluating \alg$(\psi,0)$ if the budget constraint is active}\label{sec:6.2}
Finally, we study a setting in which the active constraint is the budget one while we are introducing some tolerance on the ROI one.
We use the same setting of Experiment~\#1, except that we evaluate \alg{} and \alg$(\psi,\phi)$ algorithms.
More precisely, we relax the ROI constraint by a tolerance $\psi \in \{0, 0.05, 0.1, 0.15\}$ while keeping $\phi = 0$.
In practice, using $\psi > 0$, we allow \alg$(\psi,\phi)$ to violate the ROI constraint, but, with high probability, the violation is bounded by at most $0.5\%$, $1\%$, $1.5\%$ of $\lambda$, respectively.
Instead, we do not introduce any tolerance for the daily budget constraint $\beta$.

\paragraph{Results} Figure~\ref{fig:exp2} reports the results for Experiment \#2. For the sake of visualization, we only show the median values, on $100$ independent runs, of the performance in terms of daily revenue, ROI, and spend of \alg{} and \alg{}$(\psi, 0)$.\footnote{The $10\%$ and $90\%$ quantiles are reported in the supplementary material.}
The results show that by allowing a small tolerance in the ROI constraint violation, we can improve the exploration and, therefore, lead to faster convergence.
We note that if we set a value of $\psi \geq 0.05$, we achieve better performance in the first learning steps ($t < 20$), still maintaining a robust behavior in terms of constraint violations.
Most importantly, the ROI constraint is always satisfied by the median and also by the $10\%$ and $90\%$ quantiles.
Furthermore, a few violations are present only in the early stages of the learning process.

\section{Conclusions and Future Works}
\label{sec:7}

In this paper, we propose a novel framework for Internet advertising campaigns.
While previous works available in the literature focus only on the maximization of the revenue provided by the campaign, we introduce the concept of \emph{safety} for the algorithms choosing the bid allocation each day.
More specifically, we want that the bidding satisfies, with high probability, some daily ROI and budget constraints fixed by the business units of the companies.
Our goal is to maximize the revenue satisfying w.h.p.~the uncertain constraints (a.k.a.~safety).
We model this setting as a combinatorial optimization problem, proving that such a problem is inapproximable within any strictly positive factor unless $\mathsf{P} = \mathsf{NP}$, but it admits an exact pseudo-polynomial-time algorithm.
Most interestingly, we prove that no online learning algorithm can provide sublinear pseudo-regret while guaranteeing a sublinear number of violations of the uncertain constraints.
We show that the \textsf{GCB} algorithm suffers from a sublinear pseudo-regret, but it may violate the constraints a linear number of times. 
Thus, we design \alg{}, a novel algorithm that guarantees safety at the cost of a linear pseudo-regret.
An adaptation of \alg{}, namely \alg$(\psi,\phi)$, guarantees a sublinear pseudo-regret and safety at the cost of tolerances $\psi$ and $\phi$ on the ROI and budget constraints, respectively.
Finally, we evaluate the empirical performance of our algorithms with synthetically advertising problems that confirmed the theoretical results.

An interesting open research direction is the design of an algorithm that adopts constraints changing during the learning process, i.e., that identifies the active constraint and relaxes those that are not active.
Moreover, understanding the relationship between the relaxation of one of the constraints and the increase in revenue constitutes an interesting line of research.
%
%
\vspace{-0.2cm}
\section*{Acknowledgements}
This paper is supported by FAIR (Future Artificial Intelligence Research) project, funded by the NextGenerationEU program within the PNRR-PE-AI scheme (M4C2, Investment 1.3, Line on Artificial Intelligence).

\clearpage
\appendix
\onecolumn

\section{Omitted Proofs}\label{app:omitted}

\subsection{Proofs Omitted from Section~\ref{sec:2}} \label{app:opt}

\thminapprox*

\begin{proof}
	We restrict to the instances of SUBSET-SUM such that $z \leq \sum_{j \in S} u_j$.
	Solving these instances is trivially $\mathsf{NP}$-hard, as any instance with $z > \sum_{j \in S} u_j$ is not satisfiable, and we can decide it in polynomial time.
	Given an instance of SUBSET-SUM, let $\ell = \frac{\sum_{j \in S} u_j + 1}{\rho}$.
	Let us notice that the lower the degree of approximation we aim, the larger the value of $\ell$. 
	For instance, when studying the problem of computing an exact solution, we set $\rho =1$ and therefore $\ell = \sum_{j \in S} u_j + 1$, whereas, when we require a $1/2$-approximation, we set $\rho =1/2$ and therefore $\ell = 2(\sum_{j \in S} u_i + 1)$.
	We have $|S|+1$ subcampaigns, each denoted with $C_j$. 
	The available bid values belong to $\{0, 1\}$ for every subcampaign $C_j$.
	The parameters of the subcampaigns are set as follows.
	\begin{itemize}
		\item Subcampaign $C_0$: we set $v_0 = 1$, and 
		\[
		c_0(x)=\begin{cases}2 \ell + z~ & \textnormal{if } x= 1\\ 0 & \textnormal{otherwise}\end{cases}, \qquad n_0(x)=\begin{cases}\ell~ & \textnormal{if } x= 1\\ 0 & \textnormal{otherwise}\end{cases}. 
		\]
		\item Subcampaign $C_j$ for every $j \in S$: we set $v_j = 1$, and 
		\[
		c_j(x)=\begin{cases}u_j \textcolor{white}{\ell + z}& \textnormal{if } x= 1\\ 0 & \textnormal{otherwise}\end{cases}, \qquad n_j(x)=\begin{cases}u_j & \textnormal{if } x= 1\\ 0 & \textnormal{otherwise}\end{cases}.
		\]
	\end{itemize}
	We set the daily budget $\beta = 2(z+\ell)$ and the ROI limit $\lambda = \frac{1}{2}$.\footnote{For the sake of clarity, the proof uses simple instances. Adopting these instances is crucial to identify the most basic settings in which the problem is hard, and it is customarily done in the literature. Let us notice that it is possible to prove the theorem using more realistic instances. For example, we can build a reduction in which the costs are smaller than the values, \emph{i.e.}, $c_j(x)<n_j(x) v_j$. In particular, the reduction holds even if we set $c_0(1) = \epsilon(2l + z)$, $c_j(1) = \epsilon u_j$, $\beta = 2\epsilon(z + l)$, and $\lambda = 1/(2\epsilon)$ for an arbitrary small $\epsilon$.}
	We show that if a SUBSET-SUM instance is satisfiable, then the corresponding instance of our problem admits a solution with revenue larger than $\ell$, while if a SUBSET-SUM instance is not satisfiable, the maximum revenue in the corresponding instance of our problem is at most $\rho \, \ell-1$.
	Thus, the application of any polynomial-time $\rho$-approximation algorithm to instances of our problem generated from instances of SUBSET-SUM as described above would return a solution whose value is not smaller than $\rho\,\ell$ when the SUBSET-SUM instance is satisfiable, and it is not larger than $\rho\,\ell-1$ when the SUBSET-SUM instance is not satisfiable. 
	As a result, whenever such an algorithm returns a solution with a value that is not smaller than $\rho\,\ell$, we can decide that the corresponding SUBSET-SUM instance is satisfiable. Analogously, whenever such an algorithm returns a solution with a value that is in the range $[\rho(\rho\,\ell - 1), \rho\,\ell - 1]$, we can decide that the corresponding SUBSET-SUM instance is not satisfiable. Let us notice that the range $[\rho(\rho\,\ell - 1), \rho\,\ell - 1]$ is well defined for every $\rho \in (0,1]$, as, by construction, $\rho\,\ell = \sum_{j \in S} u_j + 1 \geq 1$ and therefore $\rho\,\ell - 1\geq \rho(\rho\,\ell - 1)$.
	Hence, such an algorithm would decide in polynomial time whether or not a SUBSET-SUM instance is satisfiable, but this is not possible unless $\mathsf{P}=\mathsf{NP}$.
	Since this holds for every $\rho \in (0,1]$, then no $\rho$-approximation to our problem is allowed in polynomial time unless $\mathsf{P}=\mathsf{NP}$.
	\textbf{If.}
	Suppose that SUBSET-SUM is satisfied by the set $S^* \subseteq S$ and that its solution assigns $x_j = 1$ if $j \in S^*$ and $x_j = 0$ otherwise, and it assigns $x_0 = 1$.
	The total revenue is $\ell + z \ge \ell$, and the constraints are satisfied.
	In particular, the sum of the costs is $2\ell + z + z = 2 (\ell + z)$, while  $\text{ROI} = \frac{\ell+z}{2\ell+2z} = \frac{1}{2}$.
	\textbf{Only if.}
	Assume by contradiction that the instance of our problem admits a solution with a revenue strictly larger than $\rho\, \ell - 1$ and that SUBSET-SUM is not satisfiable. 
	Then, it is easy to see that we need $x_0 = 1$ for campaign $C_0$ as the maximum achievable revenue is $\sum_{j \in S} u_j = \rho \,\ell-1$ when $x_0 = 0$.
	Thus, since $x_0 = 1$, the budget constraint forces $\sum_{j \in S: x_j = 1} c_i(x_j) \leq z$, thus implying $\sum_{j \in S:x_j= 1} u_j \leq z$.
	By the satisfaction of the ROI constraint, \emph{i.e.}, $\frac{\sum_{j \in S:x_j= 1} u_j+l}{\sum_{j \in S:x_j= 1} u_j+2l+z} \geq \frac{1}{2}$, it must hold $\sum_{i \in S:x_i = 1} u_i \geq z$.
	Therefore, the set $S^* = \{ i \in S : x_i = 1\}$ is a solution to SUBSET-SUM, thus reaching a contradiction.
	This concludes the proof.
\end{proof}

\subsection{Proofs Omitted from Section~\ref{sec:3}} \label{app:phipsi}

\thmtrade*

\begin{proof} 
	In what follows, we provide an impossibility result for the optimization problem in Equations~(\ref{formulation:objectivefunction})-(\ref{formulation:budgetconstraint}).
	For the sake of simplicity, our proof is based on the violation of (budget) Constraint~(\ref{formulation:budgetconstraint}), but its extension to the violation of (ROI) Constraint~(\ref{formulation:roiconstraint}) is direct.
	Initially, we show that an algorithm satisfying the two conditions of the theorem can be used to distinguish between $\mathcal{N}(1, 1)$ and $\mathcal{N}(1+\delta, 1)$ with an arbitrarily large probability using a number of samples independent from $\delta$.\footnote{
		With $\mathcal{N}(a, b)$ we denote the Gaussian distribution with mean $a$ and variance $b$.}
	Consider two instances of the bid optimization problem.
	Both instances have a single subcampaign with $x \in \{0, 1\}$, $c(0) = 0$, $r(0) = 0$, $r(1) = 1$, $\beta = 1$, and $\lambda=0$.
	The first instance has cost $c^1(1) = \mathcal{N}(1, 1)$, while the second one has cost $c^2(1) = \mathcal{N}(1+\delta, 1)$.
	With the first instance, the algorithm must choose $x = 1$ at least $T(1/2 + \epsilon)$ times in expectation; otherwise, the pseudo-regret would be strictly greater than $T(1/2 - \epsilon)$, while with the second instance, the algorithm must choose $x = 1$ at most than $T(1/2-\epsilon)$ times in expectation. Otherwise, the constraint on the budget would be violated strictly more than $T(1/2 - \epsilon)$ times.
	Standard concentration inequalities imply that, for each $\gamma > 0$, there exists a $n(\epsilon, \gamma)$ such that, given $n(\epsilon, \gamma)$ runs of the learning algorithm, with the first instance the algorithm plays $x = 1$ strictly more than $T n(\epsilon, \gamma)/2$ times with probability at least $1-\gamma$, while with the second instance, it is played strictly less than $T n(\epsilon, \gamma)/2$ times with probability at least $1-\gamma$.
	This entails that the learning algorithm can distinguish with arbitrarily large success probability (independent of $\delta$) between the two instances using (at most) $n(\epsilon, \gamma) T$ samples from one of the normal distributions.
	However, the Kullback-Leibler divergence~\citep{kullback1951information} between the two normal distributions is $KL(\mathcal{N}(1,1), \mathcal{N}(1+\delta,1)) = \delta^2/2$ and each algorithm needs at least $\Omega(1/\delta^2)$ samples to distinguish between the two distributions with arbitrarily large probability. 
	Since $\delta$ can be arbitrarily small, we have a contradiction.
	Thus, such an algorithm cannot exist. This concludes the proof.\footnote{Notice that the theorem can be modified to hold even with instances that satisfy real-world assumptions, \emph{e.g.}, with costs much smaller than the budget.
		Indeed, we can apply the same reduction in which the costs are arbitrary, \emph{e.g.}, $c(0) = c(1) = q$ with an arbitrary small $q$ and $\beta = 1$, while the utilities are $r(0) = 0$, $r(1) = \mathcal{N}(1,1)$ or $r(1) = \mathcal{N}(1 - \delta,1)$, and the ROI limit is $\lambda = 1/q$.}
\end{proof}

\subsection{Proofs Omitted from Section~\ref{sec:4}}
\optimality*

\begin{proof}
	Since all the possible values for the revenues and costs are taken into account in the subroutine, the elements in $S(y, r)$ satisfy the two inequalities in Equation~(\ref{eq:setS}) with the equal sign.
	Therefore, all the elements in $S(y, r)$ would contribute to the computation of the final value of the ROI and budget constraints, \emph{i.e.}, the ones after evaluating all the $N$ subcampaigns, with the same values for revenue and costs, being their overall revenue equal to $r$ and their overall cost equal to $y$.
	Notice that Constraint~(\ref{formulation:budgetconstraint}) is satisfied as long as it holds $\max(Y) = \beta$.
	The maximum operator in Line~\ref{line:maxim} excludes only solutions with the same costs and lower revenue, and, therefore, the subroutine excludes only solutions that would never be optimal (and, for this reason, said dominated).
	The same reasoning also holds for the subcampaign $C_1$ analysed by the algorithm.
	Finally, after all the dominated allocations have been discarded, the solution is selected in Line~\ref{line:last}, \emph{i.e.} among all the solutions satisfying the ROI constraints, the one with the largest revenue is selected.
\end{proof}

\subsection{Omitted Proofs from Section~\ref{sec:5}} \label{app:GCB}
\gcbregret*


\begin{proof}
	
	This proof extends the proof provided by~\citet{accabi2018gaussian} to the case in which multiple independent GPs are present in the optimization problem.
	
	Let us define $r_{\bm{\mu}}(\mathbf{x})$ as the expected reward provided by a specific allocation $\mathbf{x} = (x_1, \ldots, x_N)$ under the assumption that the parameter vector of the optimization problem is $\bm{\mu}$.
	Moreover, let
	\begin{align*}
		\bm{\eta} := \big[ w_1(x_1), \ldots, w_N(x_{|X_N|}), w_1(x_1), \ldots, w_N(x_{|X_N|}),
		-c_1(x_1), \ldots, - c_N(x_{|X_N|}) \big],
	\end{align*}
	be the vector characterizing the optimization problem in Equations~(\ref{formulation:objectivefunction})-(\ref{formulation:budgetconstraint}), $\mathbf{x}_t$ be the allocation chosen by the GCB algorithm at round $t$, $\mathbf{x}^*_{\bm{\eta}}$ the optimal allocation---\emph{i.e.}, the one solving the discrete version of the optimization problem in Equations~(\ref{formulation:objectivefunction})-(\ref{formulation:budgetconstraint}) with parameter $\bm{\eta}$---, and $r^*_{\bm{\eta}}$ the corresponding expected reward.
	
	To guarantee that \textsf{GCB} provides a sublinear pseudo-regret, we need a few assumptions to be satisfied.
	More specifically, we need a \emph{monotonicity property}, stating that the value of the objective function increases as the values of the elements in $\boldsymbol{\mu}$ increase and a \emph{Lipschitz continuity} assumption between the parameter vector $\boldsymbol{\mu}$ and the value returned by the objective function in Equation~(\ref{formulation:objectivefunction}).
	Formally:
	\begin{assumption}[Monotonicity] \label{ass:monotonicity}
		The expected reward $r_{\boldsymbol{\mu}}(S) := \sum_{j = 1}^N v_j \ n_j(\bi_{j,t})$, where $S$ is the bid allocation, is monotonically non decreasing in $\boldsymbol{\mu}$, i.e., given $\boldsymbol{\mu}$, $\boldsymbol{\eta}$ s.t.~$\mu_i \leq \eta_i $ for each $i$, we have $r_{\boldsymbol{\mu}}(S) \leq r_{\boldsymbol{\eta}}(S)$ for each $S$.
	\end{assumption}
	
	\begin{assumption}[Lipschitz continuity] \label{ass:lipshitz}
		The expected reward $r_{\boldsymbol{\mu}}(S)$ is Lipschitz continuous in the infinite norm w.r.t. the expected payoff vector $\boldsymbol{\mu}$, with Lipschitz constant $\Lambda > 0$.
		Formally, for each $\boldsymbol{\mu}, \boldsymbol{\eta}$ we have $|r_{\boldsymbol{\mu}}(S) - r_{\boldsymbol{\eta}}(S)| \leq \Lambda|| \boldsymbol{\mu} - \boldsymbol{\eta} ||_{\infty}$, where the infinite norm of a payoff vector is $||\boldsymbol{\mu}||_{\infty} := \max_{i} |\mu_i|$.
	\end{assumption}
	Our problem satisfies both of the above assumptions. Indeed, we have that the Lipschitz continuity holds with constant $\Lambda = v_{\max} N$.
	Instead, the monotonicity property holds by definition of $\boldsymbol{\mu}$, as the increase of a value of $\overline{w}_j(x)$ would increase the value of the objective function, and the increase of the values of $\underline{w}_j(x)$ or $\overline{c}_j(x)$ would enlarge the feasibility region of the problem, thus not excluding optimal solutions.
	
	Let us now focus on the per-step expected regret, defined as:
	$$reg_t := r^*_{\bm{\eta}} - r_{\bm{\eta}}(\mathbf{x}_t).$$
	Let us recall a property of the Gaussian distribution, which will be useful in what follows.
	Be $r \sim \mathcal{N}(0,1)$ and $c \in \mathbb{R}^+$, we have:
	\begin{align*}
		& \mathbb{P} [ r > c ] = \frac{1}{\sqrt{2 \pi}} e^{-\frac{c^2}{2}} \int_{c}^{\infty} e^{-\frac{(r - c)^{2}}{2} - c(r-c)} \mathop{dr} \\
		& \leq e^{- \frac{c^2}{2}} \mathbb{P} [ r > 0 ] = \frac{1}{2}e^{- \frac{c^2}{2}},
	\end{align*}
	since $e^{-c (r-c) } \leq 1$ for $r \geq c$. For the symmetry of the Gaussian distribution, we have:
	\begin{equation} \label{eq:normal}
		\mathbb{P} [ |r| > c ] \leq e^{- \frac{c^2}{2}}.
	\end{equation}
	
	Let us focus on the GP modeling the number of clicks.
	Following Lemma $5$ in the work by~\cite{srinivas2010gaussian}, we have that conditioned on the number of clicks $(\tilde{n}_{j,1}(x_{j,1}), \ldots, \tilde{n}_{j,t}(x_{j,t}))$, the selected bids $(x_{j,1}, \ldots, x_{j,1})$, with $x_{j,h} \in X_j$, are deterministic and the estimated number of clicks follows:
	$$ n_{j,t}(x) \sim \mathcal{N}(\hat{n}_{j,t}(x), (\hat{\sigma}^n_{j,t}(x))^2),$$
	for all $x \in X_j$.
	Thus, substituting $r = \frac{\hat{n}_{j,t}(x) - n_{j,t}(x)}{\hat{\sigma}^n_{j,t}(x)}$ and $c = \sqrt{b_t}$ in Equation~(\ref{eq:normal}), we obtain:
	\begin{equation} \label{eq:meandiff}
		\mathbb{P} \left[ |\hat{n}_{j,t}(x) - n_{j,t}(x)| > \sqrt{b_{t}} \hat{\sigma}^n_{j,t}(x) \right] \leq e^{- \frac{b_t}{2}}.
	\end{equation}
	
	Recall that, after $n$ rounds, each arm can be chosen a number of times from $1$ to $n$.
	Applying the union bound over the rounds ($h \in \{1, \ldots, T\}$), the sub-campaigns $C_j$ ($C_j$ with $j \in \{1, \ldots, N\}$), the number of times the arms in $C_j$ are chosen ($t \in \{1, \ldots, n\}$), and the available arms in $C_j$ ($x \in X_j$), and exploiting Equation~(\ref{eq:meandiff}), we obtain:
	\begin{align}
		&\mathbb{P} \left[ \bigcup_{h,j,t,x} \left( |\hat{n}_{j,t}(x) - n_{j,t}(x)| > \sqrt{b_{t}} \hat{\sigma}^n_{j,t}(x) \right) \right] \label{eq:ub1}\\
		& \leq \sum_{h = 1}^T \sum_{j = 1}^N \sum_{t = 1}^{n} |X_j| e^{- \frac{b_t}{2}}. \label{eq:ub2}
	\end{align}
	Thus, choosing $b_t = 2 \ln{ \left( \frac{\pi^2 N Q T t^2}{3\delta} \right)}$, we obtain:
	\begin{align*}
		& \sum_{h = 1}^T \sum_{j = 1}^N \sum_{t = 1}^{n} |X_j| e^{- \frac{b_t}{2}} \leq \sum_{h = 1}^T \sum_{j = 1}^N \sum_{t = 1}^{n} Q \frac{3 \delta}{\pi^2 N Q T t^2}\\	
		& \sum_{n=1}^{\infty} \frac{2\delta}{\pi^2 t^2} = \frac{\delta}{2},
	\end{align*}
	where we used the fact that $Q \geq |X_j|$ for each $j \in \{1, \ldots N\}$.
	
	Applying the same proof strategy to the GP defined over the costs leads to the following:
	\begin{equation*}
		\mathbb{P} \left[ \bigcup_{h,j,t,x} \left( |\hat{c}_{j,t}(x) - c_{j,t}(x)| > \sqrt{b_{t}} \hat{\sigma}^c_{j,t}(x) \right) \right] \leq \frac{\delta}{2}.
	\end{equation*}
	The above proof implies that the union of the event that all the bounds used in the GCB algorithm holds with probability at least $1 - \delta$. Formally, for each $t \geq 1$, we know that with probability at least $1 - \delta$ the following holds for all $x_{j} \in X_j$, $j \in \{1, \ldots N\}$, and number of times the the arm $x_j$ has been pulled over $t$ rounds:
	\begin{align} \label{eq:hypothesis}
		|\hat{n}_j(x_j) - n_j(x_j)| \leq \sqrt{b_t} \hat{\sigma}^n_{j,t}(x_j),\\
		|\hat{c}_j(x_j) - c_j(x_j)| \leq \sqrt{b_t} \hat{\sigma}^c_{j,t}(x_j).
	\end{align}
	From now on, let us assume we are in the \emph{clean event} that the previous bounds hold.
	
	Let us focus on the term $r_{\bm{\mu}}(\textbf{x}_t)$. The following holds:
	\begin{align}
		r_{\bm{\mu}}(\textbf{x}_t) &\geq r^*_{\bm{\mu}} \geq r_{\bm{\mu}}(\mathbf{x}^{*}_{\bm{\mu}}) \geq r_{\bm{\eta}}(\mathbf{x}^{*}_{\bm{\mu}}) = r^*_{\bm{\eta}}, \label{eq:monoapprox}
	\end{align}
	where we use the definition of $r^*_{\bm{\mu}}$, and the monotonicity property of the expected reward (Assumption~\ref{ass:monotonicity}), being $\left( \bm{\mu} \right)_{i} \geq \left(\bm{\eta} \right)_{i}, \forall i$.
	Using Equation~(\ref{eq:monoapprox}), the instantaneous expected pseudo-regret $reg_t$ at round $t$ satisfies the following inequality:
	\begin{align}
		reg_t &= r^*_{\bm{\eta}} - r_{\bm{\eta}}(\mathbf{x}_t) \leq r_{\bm{\mu}}(\textbf{x}_t) - r_{\bm{\eta}}(\mathbf{x}_t) = \\
		& \leq  \underbrace{r_{\bm{\mu}}(\textbf{x}_t) - r_{\hat{\bm{\mu}}}(\textbf{x}_t)}_{r_a} + \underbrace{r_{\hat{\bm{\mu}}}(\textbf{x}_t) - r_{\bm{\eta}}(\mathbf{x}_t)}_{r_b}, \label{eq:regt2}
	\end{align}
	where
	\begin{align}
		\hat{\bm{\mu}} := \big[ \hat{w}_{1,t-1}(x_1), \ldots, \hat{w}_{N,t-1}(x_{|X_N|}), & \hat{w}_{1,t-1}(x_1), \ldots, \hat{w}_{N,t-1}(x_{|X_N|}),\\ \nonumber
		& - \hat{c}_{1,t-1}(x_1), \ldots, - \hat{c}_{N,t-1}(x_{|X_N|}) \big],
	\end{align}
	is the vector composed of the estimated average payoffs for each arm $x \in X_j$ and each campaign $C_j$, where $\hat{w}_{j,t-1}(x) := v_j \hat{n}_{j,t-1}(x)$.
	
	We use the Lipschitz property of the expected reward function (see Assumption~\ref{ass:lipshitz}) to bound the terms in Equation~(\ref{eq:regt2}) as follows:
	\begin{align}
		r_a & \leq \Lambda || \bm{\mu} - \hat{\bm{\mu}} ||_{\infty} = \Lambda \max_{j \in \{1, \ldots, N\} } \left( v_{\max} \sqrt{b_t} \max_{x \in X_j} \hat{\sigma}^n_{j,t}(x) \right) \label{eq:lipsch1}\\
		& \leq N v_{\max} \sqrt{b_t} \max_{j \in \{1, \ldots, N\} } \left( \max_{x \in X_j} \hat{\sigma}^n_{j,t}(x) \right) \label{eq:lipsch4}\\
		& \leq N v_{\max} \sqrt{b_t} \sum_{j=1}^N \left( \max_{x \in X_j} \hat{\sigma}^n_{j,t}(x) \right), \label{eq:lipsch2}\\
		r_b & \leq \Lambda || \hat{\bm{\mu}} - \bm{\eta} ||_{\infty} \nonumber \\
		& \leq N v_{\max} \sqrt{b_t} \sum_{j=1}^N \left( \max_{x \in X_j} \hat{\sigma}^n_{j,t}(x) \right), \label{eq:lipsch3}
	\end{align}
	where Equation~\eqref{eq:lipsch1} holds by the definition of $\bm{\mu}$, Equation~(\ref{eq:lipsch2}) holds since the maximum over a set is not greater than the sum of the elements of the set, if they are all non-negative, and Equation~\eqref{eq:lipsch3} directly follows from Equation~\eqref{eq:hypothesis}.
	Plugging Equations~(\ref{eq:lipsch2}) and~(\ref{eq:lipsch3}) into Equation~(\ref{eq:regt2}), we obtain:
	\begin{equation} \label{eq:instregucb}
		reg_t \leq 2 N v_{\max} \sqrt{b_t} \sum_{j=1}^N \left( \max_{x \in X_j} \hat{\sigma}^n_{j,t}(x) \right).
	\end{equation}
	We need now to upper bound $\hat{\sigma}^n_{j,t}(x)$.
	Consider a realization $n_j(\cdot)$ of a GP over $X_j$ and recall that, thanks to Lemma~$5.3$ in~\citep{srinivas2010gaussian}, under the Gaussian assumption we can express the information gain $IG_{j,t}$ provided by $(\tilde{n}_j(\hat{x}_{j,1}), \ldots, \tilde{n}_j(\hat{x}_{j,|X_j|}) )$ corresponding to the sequence of arms $(\hat{x}_{j,1}, \ldots, \hat{x}_{j,|X_j|})$ as:
	\begin{equation}
		IG_{j,t} = \frac{1}{2} \sum_{h = 1}^{t} \log \left( 1 + \sigma_n^{-2}\, (\hat{\sigma}^n_{j,t}(\hat{x}_{j,h}) )^2 \right).
	\end{equation}
	We have that:
	\begin{align}
		(\hat{\sigma}^{n}_{j,t}(\hat{x}_{j,h}))^2 & = \sigma_n^{2} \left[ \sigma_n^{-2} (\hat{\sigma}^{n}_{j,t}(\hat{x}_{j,h}))^2 \right] 
		\leq \frac{\log \left[ 1 + \sigma_n^{-2} (\hat{\sigma}^{n}_{j,t}(\hat{x}_{j,h}))^2 \right]} {\log \left( 1 + \sigma_n^{-2} \right) }, \label{eq:temp1}
	\end{align}
	since $s^{2} \leq \frac{\sigma_n^{-2} \log{(1 + s^{2})}}{\log \left( 1 + \sigma_n^{-2} \right)}$ for all $s \in [0,\sigma_n^{-1}]$, and $\sigma_n^{-2} (\hat{\sigma}^{n}_{j,t}(\hat{x}_{j,h}))^2 \leq \sigma_n^{-2} \,k(\hat{x}_{j,h}, \hat{x}_{j,h}) \leq \sigma_n^{-2}$, where $k(\cdot, \cdot)$ is the kernel of the GP.
	Since Equation~(\ref{eq:temp1}) holds for any $x \in X_j$ and for any $j \in \{1, \ldots N \}$, then it also holds for the arm $\hat{x}_{\max}$ maximizing the variance $(\hat{\sigma}^{n}_{j,t}(\hat{x}_{j,h}))^2$ over $X_j$.
	Thus, setting $\bar{c} = \frac{8\, N^2}{\log \left( 1 + \sigma_n^{-2} \right)}$ and exploiting the Cauchy-Schwarz inequality, we obtain:
	\begin{align*}
		&\mathcal{R}_T^{2}(GCB) \leq T \sum_{t=1}^{T} reg^{2}_t \\
		& \leq T \sum_{t=1}^{T} 4 N^2 v^2_{\max} b_t \left[ \sum_{j=1}^N \left( \max_{x \in X_j} \hat{\sigma}^n_{j,t}(x) \right) \right]^2 \\
		& \leq 4 N^2 v^2_{\max} T b_T \sum_{t=1}^{T} \left[ N \sum_{j=1}^N \max_{x \in X_j} (\hat{\sigma}^n_{j,t}(x))^2 \right]\\ 
		& \leq \bar{c} N v^2_{\max} T b_T \sum_{j=1}^N \frac{1}{2} \sum_{t=1}^{T} \max_{x \in X_j} \log \left( 1 + \sigma_n^{-2}\, (\hat{\sigma}^n_{j,t}(\hat{x}_{j,h}) )^2 \right) \\
		& \leq \bar{c} N v^2_{\max} T b_T \sum_{j=1}^N \gamma_{j, T}.
	\end{align*}
	We conclude the proof by taking the square root on both the r.h.s.~and the l.h.s.~of the last inequality.
\end{proof}

\gcbconst*

\begin{proof}
	Let us focus on a specific day $t$.
	Consider the case in which Constraints~(\ref{formulation:roiconstraint}) and~(\ref{formulation:budgetconstraint}) are active, and, therefore, the left side equals the right side: $\sum_{j=1}^N \underline{w}_{j}(\bi_{j,t}) - \lambda \sum_{j=1}^N \ \overline{\co}_{j}(\bi_{j,t}) = 0$ and $\sum_{j=1}^{N} \overline{\co}_{j}(x_{j,t}) = \beta$.
	For the sake of simplicity, we focus on the costs $\overline{\co}_{j}(x_{j,t})$, but similar arguments also apply to the revenues $\underline{w}_j(x_{j,t})$.
	A necessary condition for which the two constraints are valid also for the actual (non-estimated) revenues and costs is that for at least one of the costs, it holds $c_j(x_{j,t}) \leq \overline{c}_{j}(x_{j,t})$.
	Indeed, if the opposite holds, \emph{i.e.,} $\overline{c}_{j}(x_{j,t}) < c_j(x_{j,t})$ for each $j \in \{1, \ldots, N\}$ and $x_{j,t} \in X_j$, the budget constraint would be violated by the allocation since $\sum_{j=1}^{N} \co_j(x_{j,t}) > \sum_{j=1}^{N} \overline{\co}_{j}(x_{j,t}) = \beta$.
	Since the event $c_j(x_{j,t}) \leq \overline{c}_{j}(x_{j,t})$ occurs with probability at most $\frac{3 \delta}{\pi^2 N Q T t^2}$, over the $t \in \mathbb{N}$, formally:
	\begin{equation*}
		\mathbb{P} \left( \frac{\sum_{j=1}^N v_j \ n_j(\realbi_{j,t}) }{\sum_{j=1}^N \ \co_j(\realbi_{j,t}) } < \lambda \vee \sum_{j=1}^{N} c_j(\realbi_{j,t}) > \beta \right) \geq 1 - \frac{3 \delta}{\pi^2 N Q T t^2}.
	\end{equation*}
	Finally, summing over the time horizon $T$, the probability that the constraints are not violated is at most $\frac{\delta}{2 N Q T}$, formally:
	\begin{equation*}
		\sum_{t=1}^T \mathbb{P} \left( \frac{\sum_{j=1}^N v_j \ n_j(\realbi_{j,t}) }{\sum_{j=1}^N \ \co_j(\realbi_{j,t}) } < \lambda \vee \sum_{j=1}^{N} c_j(\realbi_{j,t}) > \beta \right) \geq T - \frac{\delta}{2 N Q T},
	\end{equation*}
	which concludes the proof.
\end{proof}

\begin{restatable}[\textsf{GCB} cumulated violation]{thm}{gcbviolation} \label{thm:gcbviolation}
	Given $\delta \in (0,1)$, the cumulated violation of the two constraints provided by the \textsf{GCB} algorithm with probability at least $1 - \delta$, satisfies:
	\begin{itemize}
		\item $\sum_{t=1}^T \sum_{j=1}^N \co_j(\bi_{j,t}) - T \ \beta \leq \mathcal{O} \left( \sqrt{T \sum_{j=1}^N \gamma_{j,T}^c} \right)$,
		\item $T \lambda - \sum_{t=1}^T \frac{\sum_{j=1}^N v_j n_j(\bi_{j,t})}{\sum_{j=1}^N \co_j(\bi_{j,t})} \leq \mathcal{O} \left( \sqrt{T \sum_{j=1}^N(\gamma_{j,T} + \gamma_{j,T}^c)} \right)$,
	\end{itemize}
	where $\gamma_{j,t}^c$ is the maximum information gain of the GPs modeling the costs of $j$-th subcampaign after $t$ samples.
\end{restatable}

\begin{proof}
	We analyse the violation of the ROI constraint $vr_t$ at a specific day $t$ and the one of the budget constraint $vb_t$.
	
	Focusing on the budget constraint, we have:
	\begin{align}
		vb_t &= \sum_{j=1}^N \co_j(\bi_{j,t}) - \beta \leq \sum_{j=1}^N (\hat{c}_j(\bi_{j,t}) + \sqrt{b_{t-1}} \hat{\sigma}^c_{j,t-1}(\bi_{j,t})) - \beta\\
		& = \underbrace{\sum_{j=1}^N ( \hat{c}_j(\bi_{j,t}) - \sqrt{b_{t-1}} \hat{\sigma}^c_{j,t-1}(\bi_{j,t})) - \beta}_{\leq 0} + 2 \sum_{j=1}^N \sqrt{b_{t-1}} \hat{\sigma}^c_{j,t-1}(\bi_{j,t}) \label{eq:minore}\\
		& \leq 2 \sum_{j=1}^N \sqrt{b_{t-1}} \hat{\sigma}^c_{j,t-1}(\bi_{j,t}),
	\end{align}
	where the inequality in Equation~(\ref{eq:minore}) holds from the fact that the solution selected by \textsf{GCB} has to satisfy the budget constraint.
	Define $\overline{n}_j(\bi_{j,t}) := \hat{n}_j(\bi_{j,t}) + \sqrt{b_{t-1}} \hat{\sigma}_j^n(\bi_{j,t})$.
	Notice that the previous bound holds w.p.~at least $1 - \delta$ since this is the probability for which the bounds on the number of clicks and the costs hold.
	
	Since we have that with probability larger than $1-\delta$ we have that $\lambda \leq \frac{\sum_{j=1}^N v_j \overline{n}_j(\bi_{j,t})}{\sum_{j=1}^N \overline{\co}_j(\bi_{j,t})}$:
	\begin{align}
		vr_t & = \lambda - \frac{\sum_{j=1}^N v_j n_j(\bi_{j,t})}{\sum_{j=1}^N \co_j(\bi_{j,t})} \leq \frac{\sum_{j=1}^N v_j \overline{n}_j(\bi_{j,t})}{\sum_{j=1}^N \overline{\co}_j(\bi_{j,t})} - \frac{\sum_{j=1}^N v_j n_j(\bi_{j,t})}{\sum_{j=1}^N \co_j(\bi_{j,t})}\\
		&\leq \frac{\sum_{j=1}^N \co_j(\bi_{j,t}) \sum_{j=1}^N v_j \overline{n}_j(\bi_{j,t}) - \sum_{j=1}^N \overline{\co}_j(\bi_{j,t}) \sum_{j=1}^N v_j n_j(\bi_{j,t})}{\sum_{j=1}^N \co_j(\bi_{j,t}) \sum_{j=1}^N \overline{\co}_j(\bi_{j,t})}\\
		& \leq \frac{1}{N^2 c_{\min} (c_{\min} - \sqrt{b_T} \sigma_c)} \left( \sum_{j=1}^N \co_j(\bi_{j,t}) \sum_{j=1}^N v_j \overline{n}_j(\bi_{j,t}) - \sum_{j=1}^N \co_j(\bi_{j,t}) \sum_{j=1}^N v_j n_j(\bi_{j,t})\right. \nonumber\\
		& \left. + \sum_{j=1}^N \co_j(\bi_{j,t}) \sum_{j=1}^N v_j n_j(\bi_{j,t}) - \sum_{j=1}^N \overline{\co}_j(\bi_{j,t}) \sum_{j=1}^N v_j n_j(\bi_{j,t}) \right)\\
		& \leq \frac{1}{N^2 c_{\min} (c_{\min} - \sqrt{b_T} \sigma_c)} \left[ 
		\sum_{j=1}^N \co_j(\bi_{j,t}) \left( \sum_{j=1}^N v_j \overline{n}_j(\bi_{j,t}) - \sum_{j=1}^N v_j n_j(\bi_{j,t}) \right) \right. \nonumber \\
		& \left. + \sum_{j=1}^N v_j n_j(\bi_{j,t}) \left( \sum_{j=1}^N \co_j(\bi_{j,t}) - \sum_{j=1}^N \overline{\co}_j(\bi_{j,t}) \right) \right]\\
		& \leq \frac{N c_{\max} v_{\max} 2 \sum_{j=1}^N \sqrt{b_{t-1}} \hat{\sigma}_j^n(\bi_{j,t}) + N n_{\max} v_{\max} 2 \sum_{j=1}^N \sqrt{b_{t-1}} \hat{\sigma}_j^c(\bi_{j,t}) }{N^2 c_{\min} (c_{\min} - \sqrt{b_T} \sigma_c)}\\
		& = \frac{2 c_{\max} v_{\max} \sum_{j=1}^N \sqrt{b_{t-1}} \hat{\sigma}_j^n(\bi_{j,t}) + 2 n_{\max} v_{\max} \sum_{j=1}^N \sqrt{b_{t-1}} \hat{\sigma}_j^c(\bi_{j,t}) }{N c_{\min} (c_{\min} - \sqrt{b_T} \sigma_c)},
	\end{align}
	where $\sum_{j=1}^N v_j \hat{n}_j(\bi_{j,t}) \geq \sum_{j=1}^N v_j n_j(\bi_j^*)$ by definition of the \textsf{GCB} selection rule, $v_{\max} := \max_{j=1}^N v_j$, and we assume that $c_{\min} - \sqrt{b_T} \sigma_c > 0$.
	
	Using arguments similar to what has been used to bound the instantaneous regret $r_t$ in~\citet{srinivas2010gaussian} and~\citet{accabi2018gaussian}, and summing over the time horizon $T$, provides the final statement of the theorem.
\end{proof}

\gcbsaferegret*

\begin{proof}
	At the optimal solution, at least one of the constraints is active, \emph{i.e.,} it has the left-hand side equal to the right-hand side.
	Assume that the optimal clairvoyant solution $\left\{ \bi^*_j \right\}_{j=1}^N$ to the optimization problem has a value of the ROI $\lambda_{opt}$ equal to $\lambda$.
	We showed in the proof of Theorem~\ref{thm:gcbsafeconst} that for any allocation, with probability at least $1 - \frac{3 \delta}{\pi^2 N Q T t^2}$, it holds that $\frac{\sum_{j=1}^N v_j \ n_j(\bi_{j,t}) }{\sum_{j=1}^N \ \co_j(\bi_{j,t}) } > \frac{\sum_{j=1}^N v_j \ \underline{n}_j(\bi_{j,t}) }{\sum_{j=1}^N \ \overline{\co}_j(\bi_{j,t}) }$.
	This is true also for the optimal clairvoyant solution $\left\{ \bi^*_j \right\}_{j=1}^N$, for which $\lambda = \frac{\sum_{j=1}^N v_j \ n_j(\bi^*) }{\sum_{j=1}^N \ \co_j(\bi^*) } > \frac{\sum_{j=1}^N v_j \ \underline{n}_j(\bi^*) }{\sum_{j=1}^N \ \overline{\co}_j(\bi^*) }$, implying that the values used in the ROI constraint make this allocation not feasible for the \opt{} procedure.
	As shown before, this happens with probability at least $1 - \frac{3\delta}{\pi^2 N Q T t^2}$ at day $t$, and $1 - \delta$ over the time horizon $T$.
	To conclude, with probability $1 - \delta$, not depending on the time horizon $T$, we will not choose the optimal arm during the time horizon and, therefore, the regret of the algorithm cannot be sublinear.
	Notice that the same line of proof is also holding in the case the budget constraint is active, therefore, the previous result holds for each instance of the problem in Equations~(\ref{formulation:objectivefunction})-(\ref{formulation:budgetconstraint}).
\end{proof}

\gcbsafeconst*

\begin{proof}
	Let us focus on a specific day $t$.
	Constraints~(\ref{formulation:roiconstraint}) and~(\ref{formulation:budgetconstraint}) are satisfied by the solution of \opt{} for the properties of the optimization procedure.
	Define $\underline{n}_j(\bi_{j,t}) := \hat{n}_j(\bi_{j,t}) - \sqrt{b_{t-1}} \hat{\sigma}_j^n(\bi_{j,t})$.
	Thanks to the specific construction of the upper bounds, we have that $c_j(x_{j,t}) \leq \overline{c}_j(x_{j,t})$ and $n_j(x_{j,t}) \geq \underline{n}_j(x_{j,t})$, each holding with probability at least $1 - \frac{3 \delta}{\pi^2 N Q T t^2}$.
	Therefore, we have:
	\[
	\frac{\sum_{j=1}^N v_j \ n_j(\bi_{j,t}) }{\sum_{j=1}^N \ \co_j(\bi_{j,t}) } > \frac{\sum_{j=1}^N v_j \ \underline{n}_j(\bi_{j,t}) }{\sum_{j=1}^N \ \overline{\co}_j(\bi_{j,t}) } \geq \lambda
	\]
	and 
	\[
	\sum_{j=1}^{N} c_j(x_{j,t}) < \sum_{j=1}^{N} \overline{c}_j(x_{j,t}) \leq \beta.
	\]
	Using a union bound over:
	\begin{itemize}
		\item the two GPs (number of clicks and costs);
		\item the time horizon $T$;
		\item the number of times each bid is chosen in a subcampaign (at most $t$);
		\item the number of arms present in each subcampaign ($|X_j|$);
		\item the number of subcampaigns ($N$);
	\end{itemize}
	we have:
	\begin{align}
		& \sum_{t=1}^T \mathbb{P} \left( \frac{\sum_{j=1}^N v_j \ n_j(\realbi_{j,t}) }{\sum_{j=1}^N \ \co_j(\realbi_{j,t}) } < \lambda \vee \sum_{j=1}^{N} c_j(\realbi_{j,t}) > \beta \right) \leq 2 \sum_{j=1}^N \sum_{k=1}^{|X_j|} \sum_{h=1}^T \sum_{l=1}^t \frac{3 \delta}{\pi^2 N Q T l^2}\\
		& \leq 2 \sum_{j=1}^N \sum_{k=1}^{Q} \sum_{h=1}^T \sum_{l=1}^{+\infty} \frac{3 \delta}{\pi^2 N Q T l^2} = \delta.
	\end{align}
	This concludes the proof.
\end{proof}

\begin{restatable}[\alg{}$(\psi, 0)$ pseudo-regret and safety with tolerance]{thm}{gcbsaferelaxother} \label{thm:gcbsaferelaxother}
	When:
	\begin{equation*}
		\psi \geq 2 \frac{\beta_{opt} + n_{\max} }{\beta^2_{opt}} \sum_{j=1}^N v_j \sqrt{2 \ln\left(\frac{\pi^2 N Q T^3}{3 \delta'}\right)} \sigma_c \qquad \text{and} \qquad \beta_{opt}  < \beta \frac{\sum_{j=1}^N v_j}{ \frac{N \ \beta_{opt} \psi}{\beta_{opt} + n_{\max}} + \sum_{j=1}^N v_j},
	\end{equation*}
	where $\delta' \leq \delta$, $\beta_{opt}$ is the spend at the optimal solution of the original problem, and $n_{\max} := \max_{j,x} n_j(x)$ is the maximum over the sub-campaigns and the admissible bids of the expected number of clicks, \alg{}$(\psi, 0)$ provides a pseudo-regret w.r.t.~the optimal solution to the original problem of $\mathcal{O} \left( \sqrt{T \sum_{j=1}^N \gamma_{j,T}} \right)$ with probability at least $1 - \delta - \frac{\delta'}{Q T^2}$, while being $\delta$-safe w.r.t.~the constraints of the auxiliary problem.
\end{restatable}

\begin{proof}
	In what follows, we show that, at a specific day $t$, since the optimal solution of the original problem $\left\{ \bi^*_j \right\}_{j=1}^N$ is included in the set of feasible ones, we are in a setting analogous to the one of \textsf{GCB}, in which the regret is sublinear.
	Let us assume that the upper bounds on all the quantities (number of clicks and costs) hold.
	This has been shown before to occur with overall probability $\delta$ over the whole time horizon $T$.
	Moreover, notice that combining the properties of the budget of the optimal solution of the original problem $\beta_{opt}$ and using $\psi = 2 \frac{\beta_{opt} + n_{\max}}{\beta^2_{opt}} \sum_{j=1}^N v_j \sqrt{2 \ln\left(\frac{\pi^2 N Q T^3}{3 \delta'}\right)} \sigma_c$, we have:
	\begin{align}
		&\beta_{opt} < \beta \frac{\sum_{j=1}^N v_j}{ \frac{N \ \beta_{opt} \psi}{\beta_{opt} + n_{\max}} + \sum_{j=1}^N v_j}\\
		&\left( \frac{N \ \beta_{opt} \psi}{\beta_{opt} + n_{\max}} + \sum_{j=1}^N v_j \right) \beta_{opt} < \beta \sum_{j=1}^N v_j\\
		&2 N \sum_{j=1}^N v_j \sqrt{2 \ln\left(\frac{\pi^2 N Q T^3}{3 \delta'}\right)} \sigma_c + \sum_{j=1}^N v_j \beta_{opt} < \beta \sum_{j=1}^N v_j\\
		&\beta > \beta_{opt} + 2 N \sqrt{2 \ln\left(\frac{\pi^2 N Q T^3}{3 \delta'}\right)} \sigma_c. \label{eq:budopt8}
	\end{align}
	
	First, let us evaluate the probability that the optimal solution is not feasible.
	This occurs if its bounds are either violating the ROI or budget constraints.
	First, we show that analysing the budget constraint, the optimal solution of the original problem is feasible with high probability.
	Formally, it is not feasible with probability:
	\begin{align}
		&\mathbb{P} \left( \sum_{j=1}^N \overline{c}_j(\bi^*_j) > \beta \right) \leq \mathbb{P} \left( \sum_{j=1}^N \overline{c}_j(\bi^*_j) > \beta_{opt} + 2 N \sqrt{2 \ln\left(\frac{\pi^2 N Q T^3}{3 \delta'}\right)} \sigma_c \right) \label{eq:init8}\\
		&= \mathbb{P} \left( \sum_{j=1}^N \overline{c}_j(\bi^*_j) > \sum_{j=1}^N \co_j(\bi^*_j) + 2 N \sqrt{2 \ln{\frac{\pi^2 N Q T^3}{3 \delta'}}} \sigma_c \right) \\
		&\leq \sum_{j=1}^N \mathbb{P} \left( \overline{c}_j(\bi^*_j) > \co_j(\bi^*_j) + 2 \sqrt{2 \ln{\frac{\pi^2 N Q T^3}{3 \delta'}}} \sigma_C \right) \\ 
		& = \sum_{j=1}^N \mathbb{P} \left( \hat{c}_{j,t-1}(\bi^*_j) - \co_j(\bi^*_j) > - \sqrt{b_t} \hat{\sigma}^c_{j,t-1}(\bi^*_j) + 2 \sqrt{2 \ln{\frac{\pi^2 N Q T^3}{3 \delta'}}} \sigma_c \right)\\
		& \leq \sum_{j=1}^N \mathbb{P} \left( \hat{c}_{j,t-1}(\bi^*_j) - \co_j(\bi^*_j) > \sqrt{2 \ln{\frac{\pi^2 N Q T^3}{3 \delta'}}} \hat{\sigma}^c_{j,t-1}(\bi^*_j) \right) \\ 
		&\leq \sum_{j=1}^N \mathbb{P} \left( \frac{\hat{c}_{j,t-1}(\bi^*_j) - \co_j(\bi^*_j)}{\hat{\sigma}^c_{j,t-1}(\bi^*_j)} > \sqrt{2 \ln{\frac{\pi^2 N Q T^3}{3 \delta'}}} \right) \label{eq:largerbound81}\\
		& \leq \sum_{j=1}^N \frac{3 \delta'}{\pi^2 N Q T^3} = \frac{3 \delta'}{\pi^2 Q T^3}, \label{eq:endbud8}
	\end{align}
	where, in the inequality in Equation~(\ref{eq:init8}) we used Equation~(\ref{eq:budopt8}), in Equation~(\ref{eq:largerbound81}) we used the fact that $\frac{\pi^2 N Q t^2 T}{3 \delta} \leq \frac{\pi^2 N Q T^3 }{3 \delta'}$ for each $t \in \{1, \ldots, T \}$, $\hat{\sigma}^c_{j,t-1}(\bi^*_j) \leq \sigma_c$ for each $j$ and $t$, and the inequality in Equation~(\ref{eq:endbud8}) is from~\citet{srinivas2010gaussian}.
	Summing over the time horizon $T$, we get that the optimal solution of the original problem $\left\{ \bi^*_j \right\}_{j=1}^N$ is excluded from the set of the feasible ones with probability at most $\frac{3 \delta'}{\pi^2 Q T^2}$.
	
	Second, we derive a bound over the probability that the optimal solution of the original problem is feasible due to the newly defined ROI constraint.
	Let us notice that since the ROI constraint is active, we have $\lambda = \lambda_{opt}$.
	The probability that $\left\{ \bi^*_j \right\}_{j=1}^N$ is not feasible due to the ROI constraint is:
	\begin{align}
		& \mathbb{P} \left( \frac{\sum_{j=1}^N v_j \ \underline{n}_j(\bi^*_j) }{\sum_{j=1}^N \ \overline{\co}_j(\bi^*_j) } < \lambda - \psi \right) \\
		& \leq \mathbb{P} \left( \frac{\sum_{j=1}^N v_j \ \underline{n}_j(\bi^*_j) }{\sum_{j=1}^N \ \overline{\co}_j(\bi^*_j) } < \lambda_{opt} - 2 \frac{\beta_{opt} + n_{\max}}{\beta^2_{opt}} \sum_{j=1}^N v_j \sqrt{2 \ln{\frac{\pi^2 N Q T^3}{3 \delta'}}}\sigma_c \right)\\
		& = \mathbb{P} \left( \frac{\sum_{j=1}^N v_j \ \underline{n}_j(\bi^*_j) }{\sum_{j=1}^N \ \overline{\co}_j(\bi^*_j)} < \frac{\sum_{j=1}^N v_j \ n_j(\bi^*_j) }{\sum_{j=1}^N \ \co_j(\bi^*_j) } - 2\frac{\beta_{opt} + n_{\max}}{\beta^2_{opt}} \sum_{j=1}^N v_j \sqrt{2 \ln{\frac{\pi^2 N Q T^3}{3 \delta'}}}\sigma_c \right)\\
		& = \mathbb{P} \left( \sum_{j=1}^N \co_j(x^*_j) \sum_{j=1}^N v_j \ \underline{n}_j(\bi^*_j) < \sum_{j=1}^N \ \overline{\co}_j(\bi^*_j) \sum_{j=1}^N v_j \ n_j(\bi^*_j) \right. \nonumber\\
		& \qquad \left. - 2 \frac{\beta_{opt} + n_{\max}}{\beta^2_{opt}} \sum_{j=1}^N \co_j(x^*_j) \sum_{j=1}^N \ \overline{\co}_j(\bi^*_j) \sum_{j=1}^N v_j \sqrt{2 \ln{\frac{\pi^2 N Q T^3}{3 \delta'}}}\sigma_c \right)\\
		& = \mathbb{P} \left( \sum_{j=1}^N \co_j(x^*_j) \sum_{j=1}^N v_j \ \underline{n}_j(\bi^*_j) - \sum_{j=1}^N \co_j(x^*_j) \sum_{j=1}^N v_j \ n_j(\bi^*_j) +\right. \nonumber \\
		&\qquad  \frac{2}{\beta_{opt}} \sum_{j=1}^N \co_j(x^*_j) \sum_{j=1}^N \ \overline{\co}_j(\bi^*_j) \sum_{j=1}^N v_j \sqrt{2 \ln{\frac{\pi^2 N Q T^3}{3 \delta'}}}\sigma_c  \nonumber\\
		& \qquad + \sum_{j=1}^N \co_j(x^*_j) \sum_{j=1}^N v_j \ n_j(\bi^*_j) - \sum_{j=1}^N \ \overline{\co}_j(\bi^*_j) \sum_{j=1}^N v_j \ n_j(\bi^*_j) + \nonumber \\
		& \qquad \left.\frac{2 n_{\max}}{\beta^2_{opt}} \sum_{j=1}^N \co_j(x^*_j) \sum_{j=1}^N \ \overline{\co}_j(\bi^*_j)  \sum_{j=1}^N v_j \sqrt{2 \ln{\frac{\pi^2 N Q T^3}{3 \delta'}}}\sigma_c < 0 \right)\\
		& \leq \mathbb{P} \left( \sum_{j=1}^N v_j \ \underline{n}_j(\bi^*_j) - \sum_{j=1}^N v_j \ n_j(\bi^*_j) + 2 \underbrace{\frac{\sum_{j=1}^N \ \overline{\co}_j(\bi^*_j)}{\beta_{opt}}}_{\geq 1} \sum_{j=1}^N v_j \sqrt{2 \ln{\frac{\pi^2 N Q T^3}{3 \delta'}}}\sigma_c < 0 \right) \nonumber \\
		& + \mathbb{P} \left( \sum_{j=1}^N \co_j(x^*_j) \sum_{j=1}^N v_j \ n_j(\bi^*_j) - \sum_{j=1}^N \ \overline{\co}_j(\bi^*_j) \sum_{j=1}^N v_j \ n_j(\bi^*_j) \right. \nonumber \\ 
		& \qquad \left.+ 2 \underbrace{\frac{ \sum_{j=1}^N \co_j(x^*_j) \sum_{j=1}^N \ \overline{\co}_j(\bi^*_j) }{\beta^2_{opt}}}_{\geq 1} \sum_{j=1}^N v_j \underbrace{n_{\max}}_{\geq n_j(x^*_j)} \sqrt{2 \ln{\frac{\pi^2 N Q T^3}{3 \delta'}}} \sigma_c < 0 \right) \\
		& \leq \sum_{j=1}^N \mathbb{P} \left( \underline{n}_j(\bi^*_j) - n_j(\bi^*_j) + 2 \sqrt{2 \ln{\frac{\pi^2 N Q T^3}{3 \delta'}}}\sigma_c \leq 0 \right) \nonumber\\
		& \qquad + \sum_{j=1}^N \mathbb{P} \left( \co_j(x^*_j) - \overline{\co}_j(\bi^*_j) + 2 \sqrt{2 \ln{\frac{\pi^2 N Q T^3}{3 \delta'}}} \sigma_c < 0 \right) \\
		& \leq \sum_{j=1}^N \mathbb{P} \left( \hat{n}_{j,t-1}(\bi^*_j) - \sqrt{b_t} \hat{\sigma}^n_{j,t-1}(\bi^*_j) - n_j(\bi^*_j) + 2 \underbrace{\sqrt{2 \ln{\frac{\pi^2 N Q T^3}{3 \delta'}}} \sigma_c}_{\geq \sqrt{b_t} \hat{\sigma}^n_{j,t-1}(\bi^*_j)} < 0 \right) \nonumber\\
		& \qquad + \sum_{j=1}^N \mathbb{P} \left( c_{j}(\bi^*_j) - \hat{c}_{j,t-1}(\bi^*_j) - \sqrt{b_t} \hat{\sigma}^c_{j,t-1}(\bi^*_j) + 2 \underbrace{\sqrt{2 \ln{\frac{\pi^2 N Q T^3}{3 \delta'}}} \sigma_c}_{\geq \sqrt{b_t} \hat{\sigma}^c_{j,t-1}(\bi^*_j)} < 0 \right) \\
		& \leq \sum_{j=1}^N \mathbb{P} \left( n_j(\bi^*_j) < \hat{n}_{j,t-1}(\bi^*_j) + \sqrt{2 \ln{\frac{\pi^2 N Q T^3}{3 \delta'}}} \hat{\sigma}^n_{j,t-1}(\bi^*_j) \right) \nonumber\\
		& \qquad + \sum_{j=1}^N \mathbb{P} \left( c_j(\bi^*_j) < \hat{c}_{j,t-1}(\bi^*_j) - \sqrt{2 \ln{\frac{\pi^2 N Q T^3}{3 \delta'}}} \hat{\sigma}^c_{j,t-1}(\bi^*_j) \right) \label{eq:largerbound8} \\
		& = \sum_{j=1}^N \mathbb{P} \left( \frac{n_j(\bi^*_j) - \hat{n}_{j,t-1}(\bi^*_j)}{\hat{\sigma}^n_{j,t-1}(\bi^*_j)} > \sqrt{2 \ln{\frac{\pi^2 N Q T^3}{3 \delta'}}} \right) \nonumber \\
		& \qquad + \sum_{j=1}^N \mathbb{P} \left( \frac{\hat{c}_{j,t-1}(\bi^*_j) - c_j(\bi^*_j)}{\hat{\sigma}^c_{j,t-1}(\bi^*_j)} > \sqrt{2 \ln{\frac{\pi^2 N Q T^3}{3 \delta'}}} \right)\\
		& \leq 2 \sum_{j=1}^N \frac{3 \delta'}{\pi^2 N Q T^3} = \frac{6 \delta'}{\pi^2 Q T^3}, \label{eq:gaussbound8}
	\end{align}
	where in Equation~(\ref{eq:largerbound8}) we used the fact that $\frac{\pi^2 N Q t^2 T}{3 \delta} \leq \frac{\pi^2 N Q T^3 }{3 \delta'}$ for each $t \in \{1, \ldots, T \}$, $\hat{\sigma}^n_{j,t-1}(\bi^*_j) \leq \sigma_c$ for each $j$ and $t$, and the inequality in Equation~(\ref{eq:gaussbound8}) is from~\citet{srinivas2010gaussian}.
	Summing over the time horizon $T$ ensures that the optimal solution of the original problem $\left\{ \bi^*_j \right\}_{j=1}^N$ is excluded from the feasible solutions at most with probability $\frac{6 \delta'}{\pi^2 Q T^2}$.
	Finally, using a union bound, we have that the optimal solution can be chosen over the time horizon with probability at least $1 - \frac{3 \delta'}{\pi^2 Q T^2} - \frac{6 \delta'}{\pi^2 Q T^2} \leq 1 - \frac{\delta'}{Q T^2}$.
	
	Notice that here we want to compute the regret of the \alg{} algorithm w.r.t.~$\left\{ \bi^*_j \right\}_{j=1}^N$, which is not optimal for the analysed relaxed problem.
	Nonetheless, the proof on the pseudo-regret provided in Theorem~\ref{thm:gcbregret} is also valid for suboptimal solutions in the case it is feasible with high probability.
	This can be trivially shown using the fact that the regret w.r.t.~a generic solution cannot be larger than the one computed w.r.t.~the optimal one.
	Thanks to that, using a union bound over the probability that the bounds hold and that $\left\{ \bi^*_j \right\}_{j=1}^N$ is feasible, we conclude that with probability at least $1 - \delta - \frac{\delta'}{Q T^2}$ the regret \alg{} is of the order of $\mathcal{O} \left( \sqrt{T \sum_{j=1}^N \gamma_{j,T}} \right)$.
	Finally, thanks to the property of the \alg{} algorithm shown in Theorem~\ref{thm:gcbsafeconst}, the learning policy is $\delta$-safe for the relaxed problem.
\end{proof}

\begin{restatable}[\alg{}$(0, \phi)$ pseudo-regret and safety with tolerance]{thm}{gcbsaferelax} \label{thm:gcbsaferelax}
	When:
	$$\phi \geq 2 N \sqrt{2 \ln\left(\frac{\pi^2 N Q T^3}{3 \delta'}\right)} \sigma_c$$
	and
	$$\lambda_{opt} > \lambda + \frac{(\beta + n_{\max}) \phi \sum_{j=1}^N v_j}{N \beta^2},$$
	where $\delta' \leq \delta$, and $n_{\max} := \max_{j,x} n_j(x)$ is maximum expected number of clicks, \alg{}$(0, \phi)$ provides a pseudo-regret w.r.t.~the optimal solution to the original problem of $\mathcal{O} \left( \sqrt{T \sum_{j=1}^N \gamma_{j,T}} \right)$ with probability at least $1 - \delta - \frac{6 \delta'}{\pi^2 Q T^2}$, while being $\delta$-safe w.r.t.~the constraints of the auxiliary problem.
\end{restatable}

\begin{proof}
	We show that at a specific day $t$ since the optimal solution of the original problem $\left\{ \bi^*_j \right\}_{j=1}^N$ is included in the set of feasible ones, we are in a setting analogous to the one of \textsf{GCB}, in which the regret is sublinear.
	Let us assume that the upper bounds to all the quantities (number of clicks and costs) holds.
	This has been shown before to occur with overall probability $\delta$ over the whole time horizon $T$.

	First, let us evaluate the probability that the optimal solution is not feasible.
	This occurs if its bounds are either violating the ROI or budget constraints.
	From the fact that the ROI of the optimal solution satisfies $\lambda_{opt} > \lambda + \frac{(\beta + n_{\max}) \phi \sum_{j=1}^N v_j}{N \beta^2}$, we have:
	\begin{align}
		& \mathbb{P} \left( \frac{\sum_{j=1}^N v_j \ \underline{n}_j(\bi^*_j) }{\sum_{j=1}^N \ \overline{\co}_j(\bi^*_j) } < \lambda \right)  \\
		& \leq \mathbb{P} \left( \frac{\sum_{j=1}^N v_j \ \underline{n}_j(\bi^*_j) }{\sum_{j=1}^N \ \overline{\co}_j(\bi^*_j) } < \lambda_{opt} - \frac{(\beta + n_{\max}) \phi \sum_{j=1}^N v_j}{N \beta^2} \right)\\
		& = \mathbb{P} \left( \frac{\sum_{j=1}^N v_j \ \underline{n}_j(\bi^*_j) }{\sum_{j=1}^N \ \overline{\co}_j(\bi^*_j)} < \frac{\sum_{j=1}^N v_j \ n_j(\bi^*_j) }{\sum_{j=1}^N \ \co_j(\bi^*_j) } - 2\frac{\beta_{opt} + n_{\max}}{\beta^2_{opt}} \sum_{j=1}^N v_j \sqrt{\ln{\frac{\pi^2 N Q T^3}{3 \delta'}}}\sigma_c \right) \\
		& \leq \frac{3 \delta'}{\pi^2 Q T^3},
	\end{align}
	where the derivation uses arguments similar to the ones applied in the proof for the ROI constraint in Theorem~\ref{thm:gcbsaferelaxother}.
	Summing over the time horizon $T$ ensures that the optimal solution of the original problem $\left\{ \bi^*_j \right\}_{j=1}^N$ is excluded from the feasible solutions at most with probability $\frac{3 \delta'}{\pi^2 Q T^2}$.
	
	Second, let us evaluate the probability for which the optimal solution of the original problem $\left\{ \bi^*_j \right\}_{j=1}^N$ is excluded due to the budget constraint, formally:
	\begin{align}
		& \mathbb{P} \left( \sum_{j=1}^N \overline{c}_j(\bi^*_j) > \beta + \phi \right) \\
		& \leq \mathbb{P} \left( \sum_{j=1}^N \overline{c}_j(\bi^*_j) > \beta + 2 N \sqrt{2 \ln{\frac{\pi^2 N Q T^3}{3 \delta'}}} \sigma_c \right)\\
		& = \mathbb{P} \left( \sum_{j=1}^N \overline{c}_j(\bi^*_j) > \sum_{j=1}^N \ \co_j(\bi^*_j) + 2 N \sqrt{2 \ln{\frac{\pi^2 N Q T^3}{3 \delta'}}} \sigma_c \right) \\
		& \leq \sum_{j=1}^N \mathbb{P} \left( \overline{c}_j(\bi^*_j) > \co_j(\bi^*_j) + 2 \sqrt{\ln{\frac{12NT^3}{\pi^2 \delta'}}} \sigma_c \right) \label{eq:initbud} \\
		&= \sum_{j=1}^N \mathbb{P} \left( \hat{c}_{j,t-1}(\bi^*_j) - \co_j(\bi^*_j) \geq - \sqrt{b_t} \hat{\sigma}^c_{j,t-1}(\bi^*_j) + 2 \sqrt{2 \ln{\frac{\pi^2 N Q T^3}{3 \delta'}}} \sigma_c \right)\\
		& \leq \sum_{j=1}^N \mathbb{P} \left( \hat{c}_{j,t-1}(\bi^*_j) - \co_j(\bi^*_j) \geq \sqrt{2 \ln{\frac{\pi^2 N Q T^3}{3 \delta'}}} \hat{\sigma}^c_{j,t-1}(\bi^*_j) \right) \\
		& \leq \sum_{j=1}^N \mathbb{P} \left( \frac{\hat{c}_{j,t-1}(\bi^*_j) - \co_j(\bi^*_j)}{\hat{\sigma}^c_{j,t-1}(\bi^*_j)} \geq \sqrt{2 \ln{\frac{\pi^2 N Q T^3}{3 \delta'}}} \right)\\
		& \leq \sum_{j=1}^N \frac{3 \delta'}{\pi^2 N Q T^3} = \frac{3 \delta'}{\pi^2 Q T^3}, \label{eq:endbud}
	\end{align}
	where we use the fact that $\beta = \beta_{opt}$, and the derivation uses arguments similar to the ones applied in the proof for the budget constraint in Theorem~\ref{thm:gcbsaferelaxother}.
	Summing over the time horizon $T$, we get that the optimal solution of the original problem $\left\{ \bi^*_j \right\}_{j=1}^N$ is excluded from the set of the feasible ones with probability at most $\frac{\pi^2 \delta'}{6 T^2}$.
	Finally, using a union bound, we have that the optimal solution can be chosen over the time horizon with probability at least $1 - \frac{3 \delta'}{\pi^2 Q T^2}$.
	
	Notice that here we want to compute the regret of the \alg{} algorithm w.r.t.~$\left\{ \bi^*_j \right\}_{j=1}^N$ which is not optimal for the analysed relaxed problem.
	Nonetheless, the proof on the pseudo-regret provided in Theorem~\ref{thm:gcbregret} is also valid for suboptimal solutions in the case it is feasible with high probability.
	This can be trivially shown using the fact that the regret w.r.t.~a generic solution cannot be larger than the one computed on the optimal one.
	Thanks to that, using a union bound over the probability that the bounds hold and that $\left\{ \bi^*_j \right\}_{j=1}^N$ is feasible, we conclude that with probability at least $1 - \delta - \frac{6 \delta'}{\pi^2 Q T^2}$ the regret \alg{} is of the order of $\mathcal{O} \left( \sqrt{T \sum_{j=1}^N \gamma_{j,T}} \right)$.
	Finally, thanks to the property of the \alg{} algorithm shown in Theorem~\ref{thm:gcbsafeconst}, the learning policy is $\delta$-safe for the relaxed problem.
\end{proof}

\gcbsaferelaxboth*

\begin{proof}
	The proof follows from combining the arguments about the ROI constraint used in Theorem~\ref{thm:gcbsaferelaxother} and those about the budget constraint used in Theorem~\ref{thm:gcbsaferelax}.
\end{proof}

\section{Running Time} \label{app:runningtime}

The asymptotic running time of the \textsf{GCB} and \alg{} algorithms is given by the summation of the running time of the estimation and optimization subroutine.

The asymptotic running time of the estimation procedure, whose main component is the estimation of the quantities in Equations~(\ref{eq:ncmean})-(\ref{eq:cosigma}) is $\Theta(\sum_{j=1}^N |X_j| \ t^2)$, where $t$ is the number of samples (corresponding to the rounds), and the asymptotic space complexity is $\Theta(N t^2)$, \emph{i.e.}, the space required to store the Gram matrix.
A better (linear) dependence on the number of days $t$ can be obtained by using the recursive formula for the GP mean and variance computation (see~\citet{chowdhury2017kernelized} for details).

The asymptotic running time of the \textsf{Opt} procedure is $\Theta \left( \sum_{j=1}^N |X_j| \ |Y|^2 \ |R|^2 \right)$, where $|X_j|$ is the cardinality of the set of bids $X_j$, since it cycles over all the subcampaigns and, for each one of them, finds the maximum bids and compute the values in the matrix $S(y, r)$.
Moreover, the asymptotic space complexity of the \textsf{Opt} procedure is $\Theta\left(\max_{j= \{1, \ldots, N\}} |X_j| \ |Y| \ |R|\right)$ since it stores the values in the matrix $S(y, r)$ and finds the maximum over the possible bids $x \in X_j$.

\clearpage
\section{Additional Experiments and Experimental Settings Details}\label{app:expresadditional}

\subsection{Experiment~\#4: comparing \textsf{GCB}, \alg{}, and \alg{}$(\psi,\phi)$ with multiple, heterogeneous settings}\label{sec:6.4}

In this experiment, we extend the experimental activity we conduct in Experiments~\#1 and~\#3 to other multiple heterogeneous settings.

\paragraph{Setting}
We simulate $N = 5$ subcampaigns with a daily budget $\beta = 100$, with $|X_j| = 201$ bid values evenly spaced in $[0, \ 2]$, $|Y| = 101$ cost values evenly spaced in $[0, \ 100]$, being the daily budget $\beta = 100$, and $|R|$ evenly spaced revenue values depending on the setting.
We generate $10$ scenarios that differ in the parameters defining the cost and revenue functions and in the ROI parameter $\lambda$.
Recall that the number-of-click functions coincide with the revenue functions since $v_j = 1$ for each $j \in \{1, \dots, N\}$. 
Parameters $\alpha_j \in \mathbb{N}^+$ and $\theta_j \in \mathbb{N}^+$ are sampled from discrete uniform distributions $\mathcal{U}\{50, 100\}$ and $\mathcal{U}\{400, 700\}$, respectively. 
Parameters $\gamma_j$ and $\delta_j$ are sampled from the continuous uniform distributions $\mathcal{U}[0.2 , 1.1)$.
Finally, parameters $\lambda$ are chosen such that the ROI constraint is active at the optimal solution.
Table~\ref{tab:param_10settings} specifies the values of such parameters.

\begin{table}
	\centering
	\caption{Values of the parameters used in the $10$ different settings of Experiment~\#4.}
	\label{tab:param_10settings}
	\setlength{\extrarowheight}{0pt}
	\renewcommand{\arraystretch}{0.85}
			\begin{tabular}{l|l||c|c|c|c|c|c|}
				

				&& $C_1$ & $C_2$ & $C_3$ & $C_4$ & $C_5$ & $\lambda$\\
				\hline\hline
				
				Setting 1 & $\theta_j$ & 530  & 417 & 548 & 571 & 550 &  10.0  \\ \cline{2-7}
				& $\delta_j$  & 0.356 & 0.689 & 0.299 & 0.570 & 0.245 &    \\ \cline{2-7}
				& $\alpha_j$  & 83 & 97 & 72 & 100 & 96 &    \\ \cline{2-7}
				& $\gamma_j$ & 0.939 & 0.856 & 0.484 & 0.661 & 0.246 &   \\ \hline\hline

				Setting 2 & $\theta_j$ & 597 & 682 & 698  & 456 & 444 &  14.0  \\ \cline{2-7}
				& $\delta_j$  & 0.202 & 0.520 & 0.367 & 0.393 & 0.689 &    \\ \cline{2-7}
				& $\alpha_j$  & 83 & 98 & 56 & 60 & 51 &    \\ \cline{2-7}
				& $\gamma_j$ & 0.224 & 0.849 & 0.726 & 0.559 & 0.783 &   \\ \hline\hline
				
				Setting 3 &  $\theta_j$ & 570 & 514 & 426 & 469 & 548 & 10.5 \\ \cline{2-7}
				& $\delta_j$  & 0.217 & 0.638 & 0.694 & 0.391 & 0.345 &    \\ \cline{2-7}
				& $\alpha_j$  & 97 & 78 & 53 & 80 & 82 &    \\ \cline{2-7}
				& $\gamma_j$ & 0.225 & 0.680 & 1.051 & 0.412 & 0.918 &   \\ \hline\hline
				
				Setting 4 &  $\theta_j$ & 487 & 494 & 467 & 684 & 494 & 12.0 \\ \cline{2-7}
				& $\delta_j$  & 0.348 & 0.424 & 0.326 & 0.722 & 0.265 &    \\ \cline{2-7}
				& $\alpha_j$  & 62 & 79 & 76 & 69 & 99 &    \\ \cline{2-7}
				& $\gamma_j$ & 0.460 & 1.021 & 0.515 & 0.894 & 1.056 &   \\ \hline\hline
				
				Setting 5 & $\theta_j$ & 525 & 643 & 455 & 440 & 600 & 14.0 \\ \cline{2-7}
				& $\delta_j$  & 0.258 & 0.607 & 0.390 & 0.740 & 0.388 &    \\ \cline{2-7}
				& $\alpha_j$  & 52 & 87 & 68 & 99 & 94 &    \\ \cline{2-7}
				& $\gamma_j$ & 0.723 & 0.834 & 1.054 & 1.071 & 0.943 &   \\ \hline\hline
				
				Setting 6 & $\theta_j$ & 617 & 518 & 547 & 567 & 576 & 11.0 \\ \cline{2-7}
				& $\delta_j$  & 0.844 & 0.677 & 0.866 & 0.252 & 0.247 &    \\ \cline{2-7}
				& $\alpha_j$  & 71 & 53 & 87 & 98 & 59 &    \\ \cline{2-7}
				& $\gamma_j$ & 0.875 & 0.841 & 1.070 & 0.631 & 0.288 &   \\ \hline\hline
				
				Setting 7 & $\theta_j$ & 409 & 592 & 628 & 613 & 513 & 11.5 \\ \cline{2-7}
				& $\delta_j$  & 0.507 & 0.230 & 0.571 & 0.359 & 0.307 &     \\ \cline{2-7}
				& $\alpha_j$  & 77 & 78 & 91 & 50 & 71 &    \\ \cline{2-7}
				& $\gamma_j$ & 0.810 & 0.246 & 0.774 & 0.516 & 0.379 &   \\ \hline\hline
				
				Setting 8 & $\theta_j$ & 602 & 605 & 618 & 505 & 588 &  13.0 \\ \cline{2-7}
				& $\delta_j$  & 0.326 & 0.265 & 0.201 & 0.219 & 0.291 &    \\ \cline{2-7}
				& $\alpha_j$  & 67 & 80 & 99 & 77 & 99 &    \\ \cline{2-7}
				& $\gamma_j$ & 0.671 & 0.775 & 0.440 & 0.310 & 0.405 &   \\ \hline\hline
				
				Setting 9 & $\theta_j$ & 486 & 684 & 547 & 419 & 453 & 13.0 \\ \cline{2-7}
				& $\delta_j$  & 0.418 & 0.330 & 0.529 & 0.729 & 0.679 &    \\ \cline{2-7}
				& $\alpha_j$  & 53 & 82 & 58 & 96 & 100 &    \\ \cline{2-7}
				& $\gamma_j$ & 0.618 & 0.863 & 0.669 & 0.866 & 0.831 &   \\ \hline\hline
				
				Setting 10 & $\theta_j$ & 617 & 520 & 422 & 559 & 457 & 14.0 \\ \cline{2-7}
				& $\delta_j$  & 0.205 & 0.539 & 0.217 & 0.490 & 0.224 &    \\ \cline{2-7}
				& $\alpha_j$  & 51 & 86 & 93 & 61 & 84 &    \\ \cline{2-7}
				& $\gamma_j$ & 1.0493 & 0.779 & 0.233 & 0.578 & 0.562 &   \\ \hline
			\end{tabular}
	\end{table}
	
\setlength{\extrarowheight}{4pt}
\renewcommand{\arraystretch}{0.7}
\begin{table*}
	\caption{Results of Experiment~\#4.} \label{tab:rand_performance}
	\centering
	\resizebox{\columnwidth}{!}{\begin{tabular}{@{\hspace{4pt}}l@{\hspace{4pt}}|@{\hspace{4pt}}l@{\hspace{0pt}}||@{\hspace{4pt}}c@{\hspace{4pt}}|@{\hspace{4pt}}c@{\hspace{4pt}}|@{\hspace{4pt}}c@{\hspace{4pt}}|@{\hspace{4pt}}c@{\hspace{4pt}}|@{\hspace{4pt}}c@{\hspace{4pt}}|@{\hspace{4pt}}c@{\hspace{4pt}}|@{\hspace{4pt}}c@{\hspace{4pt}}|@{\hspace{4pt}}c@{\hspace{4pt}}|@{\hspace{4pt}}c@{\hspace{4pt}}|@{\hspace{4pt}}c@{\hspace{4pt}}|@{\hspace{4pt}}c@{\hspace{4pt}}|@{\hspace{4pt}}c@{\hspace{4pt}}|}
			\multicolumn{2}{c||}{}& $W_T$ & $W_{T/2}$ & $\sigma_T$ & $\sigma_{T/2}$ & $M_T$ & $M_{T/2}$ & $U_T$ & $U_{T/2}$ & $L_T$ & $L_{T/2}$ & $V_{ROI}$ & $V_{B}$\\
			\midrule			
			\multirow{4}{*}{\rotatebox[origin=c]{90}{Setting $\#1$}} & \textsf{GCB} & 57481 & 30767 & 556 & 376 & 57497 & 30811 & 58081 & 31239 & 56758 & 30288 & 1.00 & 0.62\\
			\cline{2-14}
			& \alg{}  & 44419 & 21549 & 4766 & 2474 & 45348 & 21972 & 46783 & 23163 & 42287 & 20324 & 0.02& 0.00\\
			\cline{2-14}
			& \alg{}$(0.05,0)$ & 48028 & 23524 & 4902 & 2487 & 48626 & 23831 & 50388 & 24827 & 46307 & 22506 & 0.21 & 0.00\\
			\cline{2-14}
			& \alg{}$(0.10,0)$ & 52327 & 25859 & 829 & 611 & 52338 & 25887 & 53324 & 26605 & 51316 & 25104 & 0.94 & 0.00\\
			\hline \hline
			\multirow{4}{*}{\rotatebox[origin=c]{90}{Setting $\#2$}} & \textsf{GCB} & 63664 & 35566 & 1049 & 679 & 63701 & 35573 & 64984 & 36524 & 62249 & 34675 & 1.00 & 0.14\\
			\cline{2-14}
			& \alg{}  & 34675 & 16290 & 8541 & 4448 & 37028 & 17647 & 39594 & 19473 & 27748 & 11141 & 0.03 & 0.00\\
			\cline{2-14}
			& \alg{}$(0.05,0)$  & 40962 & 19564 & 6013 & 3122 & 41823 & 20152 & 44468 & 21698 & 38640 & 17645 & 0.04 & 0.00\\
			\cline{2-14}
			& \alg{}$(0.10,0)$ & 46694 &  22099 & 6382 & 3112 & 47749 & 22433 & 51564 & 24776 & 44099 & 19929 & 0.72 & 0.00\\
			\hline \hline
			\multirow{4}{*}{\rotatebox[origin=c]{90}{Setting $\#3$}} & \textsf{GCB}  & 54845 & 30213& 757 & 478 & 54816 & 30177 & 55734 & 30885 & 54006 & 29638 & 1.00 & 0.25\\
			\cline{2-14}
			& \alg{}  & 35726 & 16577 & 8239 & 4361 & 38302 & 18114 & 40746 & 19882 & 27279 & 8791 & 0.03 & 0.00\\
			\cline{2-14}
			& \alg{}$(0.05,0)$ & 38757 & 18370 & 8492 & 4594 & 41422 & 19808 & 43337 & 21092 & 30413 & 12678 & 0.07 & 0.00\\
			\cline{2-14}
			& \alg{}$(0.10,0)$ & 42184 & 19993 & 9652 & 5056 & 44820 & 21574 & 47659 & 23118 & 36570 & 14450 & 0.75 & 0.00\\
			\hline \hline			
			\multirow{4}{*}{\rotatebox[origin=c]{90}{Setting $\#4$}} & \textsf{GCB} & 71404 & 37383 & 351 & 262 & 71399 & 37387 & 71877 & 37732 & 70930 & 37021 & 0.98 & 0.98\\
			\cline{2-14}
			& \alg{}  & 29101 & 13817 & 7052 & 3646 & 30992 & 14680 & 35602 & 17256 & 20509 & 9562 & 0.00 & 0.00\\
			\cline{2-14}
			& \alg{}$(0.05,0)$  & 39802 & 18270 & 10232 & 4955 & 38296 & 17994 & 53375 & 24962 & 25197 & 11341 & 0.01 & 0.00\\
			\cline{2-14}
			& \alg{}$(0.10,0)$ & 51515 & 24095 & 11094 & 5639 & 56621 & 24902 & 61992 & 30020 & 35642 & 16198 & 0.56 & 0.00\\
			\hline \hline
			\multirow{4}{*}{\rotatebox[origin=c]{90}{Setting $\#5$}} & \textsf{GCB} & 74638 & 39523 & 642 & 392 & 74693 & 39529 & 75405 & 40049 & 73756 & 39063 & 0.98 & 0.31\\
			\cline{2-14}
			& \alg{}  & 48956 & 23230 & 6715 & 3486 & 50021 & 23838 & 53644 & 26266 & 42946 & 19287 & 0.00 & 0.00 \\
			\cline{2-14}
			& \alg{}$(0.05,0)$  & 56205 & 27003 & 2578 & 1742 & 56554 & 27211 & 58839 & 28802 & 53278 & 24987 & 0.00 & 0.00\\
			\cline{2-14}
			& \alg{}$(0.10,0)$ & 63411 & 30207 & 5636 & 2916 & 64364 & 30665 & 66764 & 32212 & 60519 & 28260 & 0.59 & 0.00\\
			\hline \hline
			\multirow{4}{*}{\rotatebox[origin=c]{90}{Setting $\#6$}} & \textsf{GCB} & 67118 & 35775 & 327 & 260 & 67130 & 35795 & 67536 & 36111 & 66726 & 35424 & 0.98 & 0.98\\
			\cline{2-14}
			& \alg{}  & 14448 & 7707 & 6006 & 3065 & 15019 & 8075 & 18581 & 9800 & 6781 & 3926 & 0.02 & 0.00\\
			\cline{2-14}
			& \alg{}$(0.05,0)$ & 14968 & 7710 & 6174 & 2974 & 15161 & 8157 & 20548 & 10351 & 7954 & 3860 & 0.02 & 0.00\\
			\cline{2-14}
			& \alg{}$(0.10,0)$ & 34716 & 15507 & 16133 & 7280 & 37409 & 16601 & 55236 & 25366 & 9895 & 5188 & 0.19 & 0.00\\
			\hline \hline
			\multirow{4}{*}{\rotatebox[origin=c]{90}{Setting $\#7$}} & \textsf{GCB} & 63038 & 35330 & 873 & 401 & 63088 & 35367 & 64226 & 35793 & 61754 & 34823 & 1.00 & 0.41\\
			\cline{2-14}
			& \alg{}  & 31662 & 14806 & 5651 & 3090 & 33009 & 15570 & 35004 & 16922 & 28296 & 11338 & 0.04 & 0.00\\
			\cline{2-14}
			& \alg{}$(0.05,0)$ & 37744 & 17606 & 4173 & 2619 & 38321 & 18161 & 41184 & 19805 & 33914 & 15276 & 0.03 & 0.00\\
			\cline{2-14}
			& \alg{}$(0.10,0)$ & 42528 & 20046 & 7497 & 3624 & 43765 & 20683 & 47187 & 22301 & 38988 & 18314 & 0.70 & 0.00\\
			\hline \hline
			\multirow{4}{*}{\rotatebox[origin=c]{90}{Setting $\#8$}} & \textsf{GCB} & 79571 & 42322 & 476 & 375 & 79581 & 42317 & 80073 & 42743 & 78969 & 41913 & 1.00 & 0.98\\
			\cline{2-14}
			& \alg{} & 48046 & 22478 & 11779 & 6000 & 52094 & 24180 & 57321 & 28024 & 30655 & 13338 & 0.02 & 0.00\\
			\cline{2-14}
			& \alg{}$(0.05,0)$ & 58450 & 27477 & 10296 & 5605 & 61404 & 28845 & 66902 & 32883 & 41196 & 18222 & 0.02 & 0.00\\
			\cline{2-14}
			& \alg{}$(0.10,0)$ & 68252 & 33255 & 3436 & 2417 & 68886 & 33857 & 70758 & 35377 & 65394 & 30696 & 0.07 & 0.00\\
			\hline \hline
			\multirow{4}{*}{\rotatebox[origin=c]{90}{Setting $\#9$}} & \textsf{GCB} & 70280 & 37363 & 672 & 347 & 70275 & 37352 & 71123 & 37811 & 69379 & 36942 & 1.00 & 0.34\\
			\cline{2-14}
			& \alg{}  & 40116 & 18895 & 5522 & 3047 & 40673 & 19357 & 43850 & 21161 & 37310 & 17222 & 0.03 & 0.00  \\
			\cline{2-14}
			& \alg{}$(0.05,0)$ & 51138 & 23683 & 3110 & 2036 & 50984 & 23375 & 54545 & 26174 & 47465 & 21385 & 0.03 &  0.00\\
			\cline{2-14}
			& \alg{}$(0.10,0)$ & 63574 & 29675 & 3810 & 3323 & 64011 & 30112 & 66658 & 32559 & 60970 & 27280 & 0.80 & 0.00\\
			\hline \hline
			\multirow{4}{*}{\rotatebox[origin=c]{90}{Setting $\#10$}} & \textsf{GCB} & 80570 & 41973 & 435 & 344 & 80568 & 42019 & 81127 & 42388 & 80023 & 41496 & 1.00 & 0.98\\ 
			\cline{2-14}
			& \alg{} & 58965 & 28785 & 3097 & 1465 & 60033 & 28917 & 62353 & 30535 & 54590 & 26931 & 0.02 & 0.00\\
			\cline{2-14}
			& \alg{}$(0.05,0)$  & 63685 & 31004 & 3787 & 1876 & 65273 & 31550 & 67364 & 33105 & 57860 & 28349 & 0.02 & 0.00\\
			\cline{2-14}
			& \alg{}$(0.10,0)$ & 68480 & 33358 & 4224 & 2181 & 70388 & 33998 & 72730 & 35838 & 61971 & 30317 & 0.65 & 0.00  \\
			\bottomrule
	\end{tabular}}
\end{table*}

	\paragraph{Results}
	We compare the \textsf{GCB}, \alg{}, \alg$(0.05, 0)$, and \alg$(0.10, 0)$ algorithms in terms of:
	\begin{itemize}
		\item $W_t := \sum_{h=1}^t P_t(\mathfrak{U})$: average (over $100$ runs) cumulative revenue at round $t$ (and the corresponding standard deviation $\sigma_t$);
		\item $M_t$: median (over $100$ runs) of the cumulative revenue at round $t$;
		\item $U_t$: $90$-th percentile (over $100$ runs) of the cumulative revenue at round $t$;
		\item $L_t$: $10$-th percentile (over $100$ runs) of the cumulative revenue at round $t$;
		\item $V_{ROI}$: the fraction of days in which the ROI constraint is violated;
		\item $V_{B}$: the fraction of days in which the budget constraint is violated.
	\end{itemize}
	
	%
	%


	Table~\ref{tab:rand_performance} reports the algorithms' performances at $\lceil T/2 \rceil = 28$ and at the end of the learning process $t = T = 57$.
	As already observed in the previous experiments, \textsf{GCB} violates the ROI constraint at every round, run, and setting. More surprisingly, \textsf{GCB} violates the budget constraint most of the time (60\% on average), even if that constraint is not active at the optimal solution.
	Interestingly, \alg{}$(\psi,0)$ never violates the budget constraints (for every $\psi$). As expected, the violation of the ROI constraint is close to zero with \alg{}, while it increases as $\psi$ increases.
	In terms of average cumulative revenue, at $T$, we observe that \alg{} gets about $56\%$ of the revenue provided by \textsf{GCB}, while the ratio related to \alg{}$(0.05,0)$ is about $66\%$ and that related to \alg{}$(0.10,0)$ is about $78\%$. At $T/2$, we the ratios are about $52\%$ for \textsf{GCB}, $61\%$ for \alg{}$(0.05,0)$, and $73\%$ for \alg{}$(0.10,0)$, showing that those ratios increase as $T$ increases. The rationale is that in the early stages of the learning process, safe algorithms learn more slowly than non-safe algorithms.
	Similar performances can be observed when focusing on the other indices.
	Summarily, the above results show that our algorithms provide advertisers with a wide spectrum of effective tools to address the revenue/safety tradeoff.
	A small value of $\psi$ (and $\phi$) represents a good tradeoff.
	The choice of the specific configuration to adopt in practice depends on the advertiser's aversion to violating the constraints.
	
	\subsection{Additional Information for Reproducibility}\label{app:reproducibility}
	
	In this section, we provide additional information for the full reproducibility of the experiments provided in the main paper.
	
	The code has been run on an Intel(R) Core(TM) $i7-4710MQ$ CPU with $16$ GiB of system memory.
	The operating system was Ubuntu $18.04.5$ LTS, and the experiments were run on Python $3.7.6$.
	The libraries used in the experiments, with the corresponding versions, were:
	\begin{itemize}
		\item \texttt{matplotlib==3.1.3}
		\item \texttt{gpflow==2.0.5}
		\item \texttt{tikzplotlib==0.9.4}
		\item \texttt{tf\_nightly==2.2.0.dev20200308}
		\item \texttt{numpy==1.18.1}
		\item \texttt{tensorflow\_probability==0.10.0}
		\item \texttt{scikit\_learn==0.23.2}
		\item \texttt{tensorflow==2.3.0}
	\end{itemize}
	
	On this architecture, the average execution time of each algorithm takes an average of $\approx 30$ sec for each day $t$ of execution.
	
	\subsection{Parameters and Setting of Experiment \#1 (main paper)} \label{app:tech}
	Table \ref{tab:params_exp1} specifies the values of the parameters of cost and number-of-click functions of the subcampaigns used in Experiment~\#1.
	
	\begin{table}[t!]
		\caption{Parameters of the synthetic settings used in Experiment~\#1.} 
		\large
		\centering
		\begin{tabu}{|l|[2pt]l|l|l|l|l|}
			\hline
			& $C_1$  & $C_2$ & $C_3$ & $C_4$ & $C_5$\\ \tabucline[2pt]{-}
			$\theta_j$  & 60  & 77 & 75  & 65  & 70 \\  \hline
			$\delta_j$ & 0.41    & 0.48   & 0.43   & 0.47 & 0.40 \\ \hline
			$\alpha_j$ & 497  & 565 & 573  & 503  & 536  \\  \hline
			$\gamma_j$ & 0.65    & 0.62   & 0.67   & 0.68 & 0.69\\  \hline
			$\sigma_f$ GP revenue &  0.669   & 0.499  & 0.761  &  0.619& 0.582   \\  \hline
			$l$ GP revenue &  0.425   &0.469 & 0.471  & 0.483 & 0.386   \\  \hline
			$\sigma_f$ GP cost & 0.311   &  0.443  & 0.316  &  0.349 & 0.418   \\  \hline
			$l$ GP cost & 0.76   & 0.719  & 0.562  & 0.722 & 0.727   \\  \hline
		\end{tabu}
		\label{tab:params_exp1}
	\end{table}
	
	\clearpage
	\subsection{Additional Results of Experiment \#2} \label{app:addfig3}
	
	In Figures~\ref{fig:exp3_1}, \ref{fig:exp3_2}, and~\ref{fig:exp3_3} we report the $90\%$ and $10\%$ of the quantities analysed in the experimental section for Experiment~\#3 provided by the \textsf{GCB}, \alg, and \alg{}$(0.05,0)$, respectively.
	These results show that the constraints are satisfied by \alg, and \alg{}$(0.05,0)$ also with high probability.
	While for \alg{} this is expected due to the theoretical results we provided, the fact that also \alg{}$(0.05,0)$ guarantees safety w.r.t.~the original optimization problem suggests that in some specific setting \alg{} is too conservative.
	This is reflected in a lower cumulative revenue, which might be negative from a business point of view.
	
	\begin{figure*}[th!]
		\centering
		\captionsetup[subfigure]{oneside,margin={1.05cm,0cm}}
		\subfloat[]{\scalebox{0.67}{\includegraphics{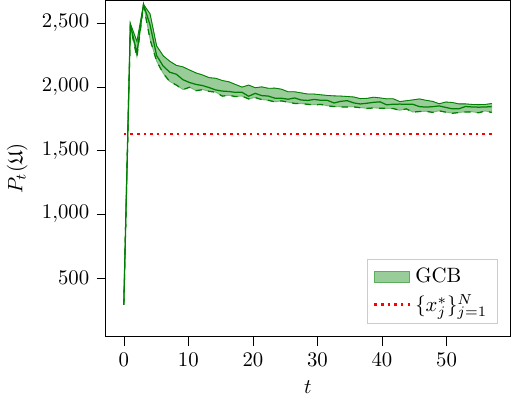}\label{fig:add1}}}
		\subfloat[]{\scalebox{0.67}{\includegraphics{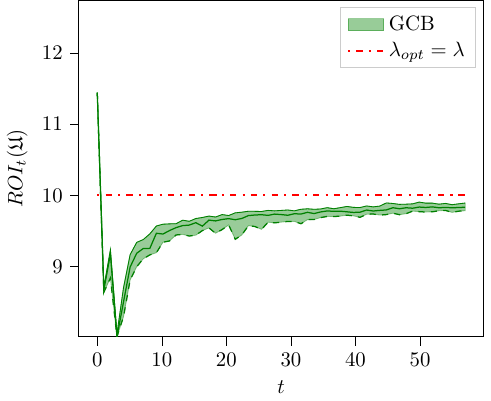}\label{fig:add2}}}
		\subfloat[]{\scalebox{0.67}{\includegraphics{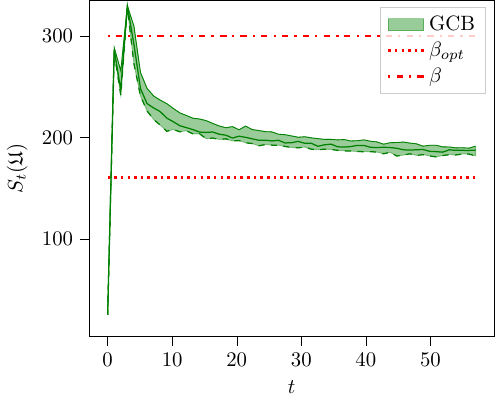}\label{fig:add3}}}
		\caption{Results of Experiment \#2: daily revenue (a), ROI (b), and spend (c) obtained by \textsf{GCB}. The dash-dotted lines correspond to the optimum values for the revenue and ROI, while the dashed lines correspond to the values of the ROI and budget constraints.}
		\label{fig:exp3_1}
	\end{figure*}

	\begin{figure*}[th!]
		\centering
		\captionsetup[subfigure]{oneside,margin={1.05cm,0cm}}
		\subfloat[]{\scalebox{0.67}{\includegraphics{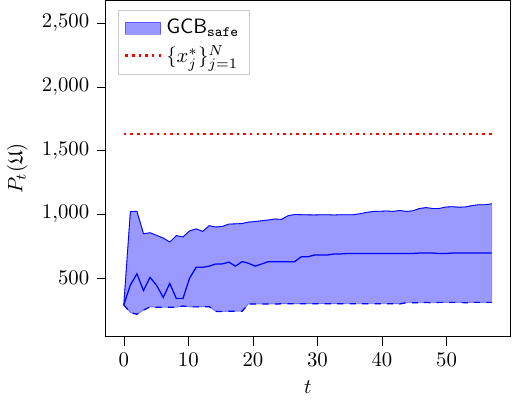}\label{fig:add4}}}
		\subfloat[]{\scalebox{0.67}{\includegraphics{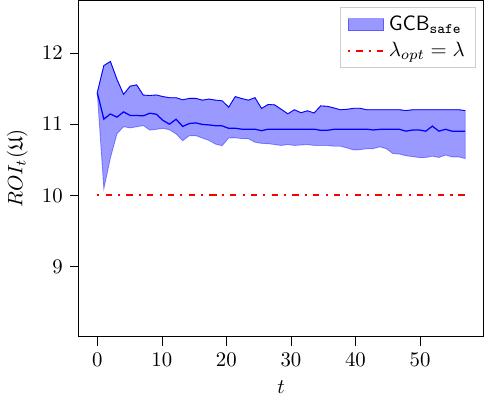}\label{fig:add5}}}
		\subfloat[]{\scalebox{0.67}{\includegraphics{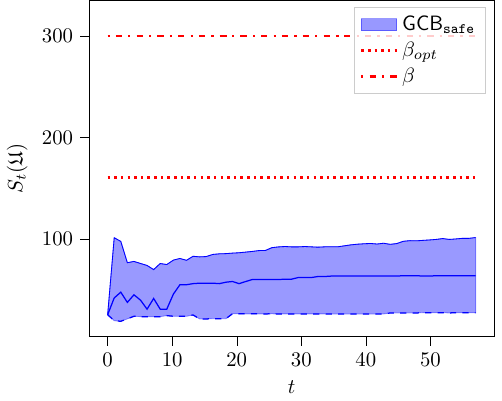}\label{fig:add6}}}		
		\caption{Results of Experiment \#2: daily revenue (a), ROI (b), and spend (c) obtained by \alg{}. The dash-dotted lines correspond to the optimum values for the revenue and ROI, while the dashed lines correspond to the values of the ROI and budget constraints.}
		\label{fig:exp3_2}
	\end{figure*}
	
	\begin{figure*}[th!]
		\centering
		\captionsetup[subfigure]{oneside,margin={1.05cm,0cm}}
		\subfloat[]{\scalebox{0.67}{\includegraphics{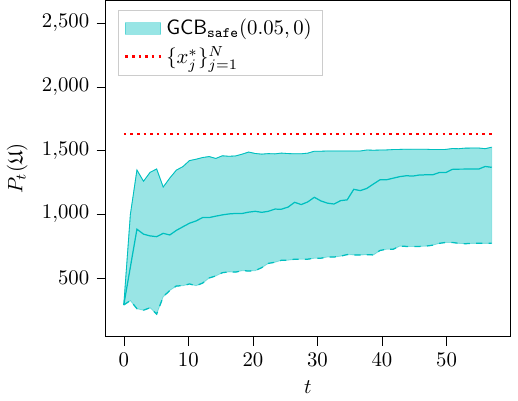}\label{fig:add7}}}
		\subfloat[]{\scalebox{0.67}{\includegraphics{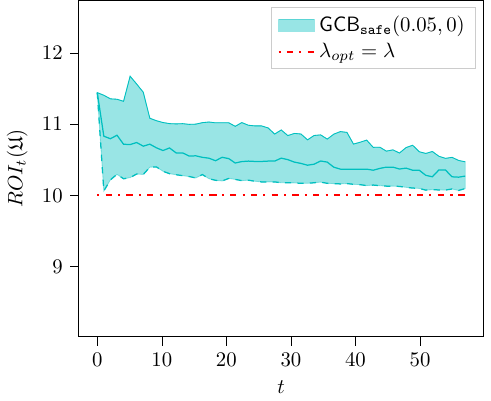}\label{fig:add8}}}
		\subfloat[]{\scalebox{0.67}{\includegraphics{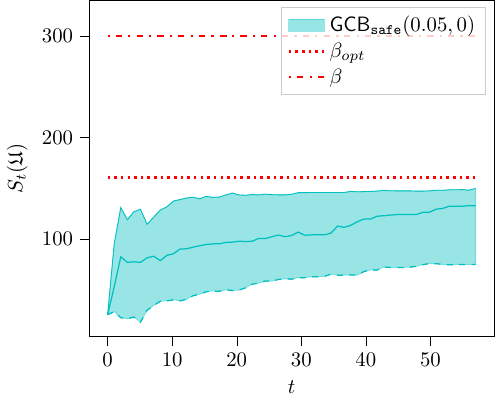}\label{fig:add9}}}		
		\caption{Results of Experiment \#2: daily revenue (a), ROI (b), and spend (c) obtained by \alg{}$(0.05,0)$. The dash-dotted lines correspond to the optimum values for the revenue and ROI, while the dashed lines correspond to the values of the ROI and budget constraints.}
		\label{fig:exp3_3}
	\end{figure*}
	
	\clearpage
	\subsection{Additional Results of Experiment \#3} \label{app:addfig}
	
	In Figures~\ref{fig:add_exp2_1}, \ref{fig:add_exp2_2}, \ref{fig:add_exp2_3}, and~\ref{fig:add_exp2_4} we report the $90\%$ and $10\%$ of the quantities related to Experiment~\#2 provided by the \alg, \alg{}$(0,0.05)$, \alg{}$(0,0.10)$, and \alg{}$(0,0.15)$, respectively.
	
\begin{figure*}[th!]
	\centering
	\captionsetup[subfigure]{oneside,margin={1.05cm,0cm}}
	\subfloat[]{\scalebox{0.67}{\includegraphics{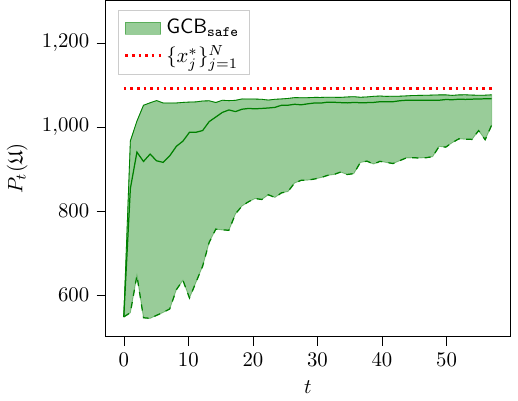}\label{fig:addx4}}}
	\subfloat[]{\scalebox{0.67}{\includegraphics{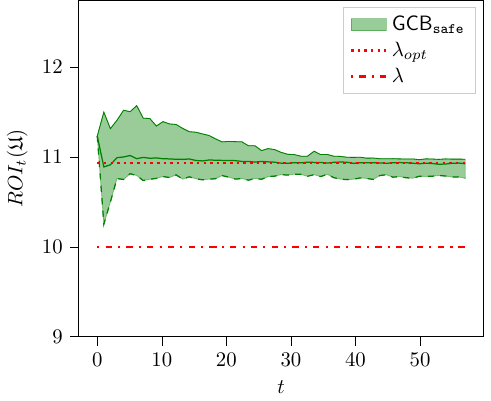}\label{fig:addx5}}}
	\subfloat[]{\scalebox{0.67}{\includegraphics{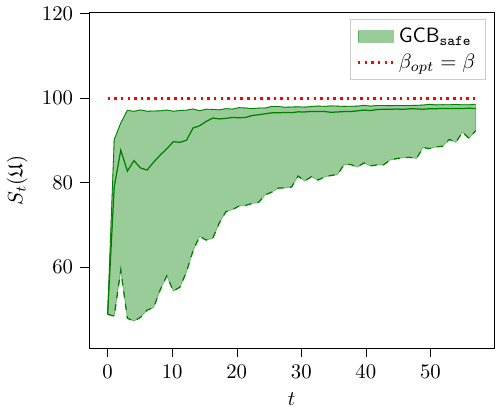}\label{fig:addx6}}}
	\caption{Results of Experiment \#3: daily revenue (a), ROI (b), and spend (c) obtained by \alg{}. The dash-dotted lines correspond to the optimum values for the revenue and ROI, while the dashed lines correspond to the values of the ROI and budget constraints.}
	\label{fig:add_exp2_1}
\end{figure*}

\begin{figure*}[th!]
	\centering
	\captionsetup[subfigure]{oneside,margin={1.05cm,0cm}}
	\subfloat[]{\scalebox{0.67}{\includegraphics{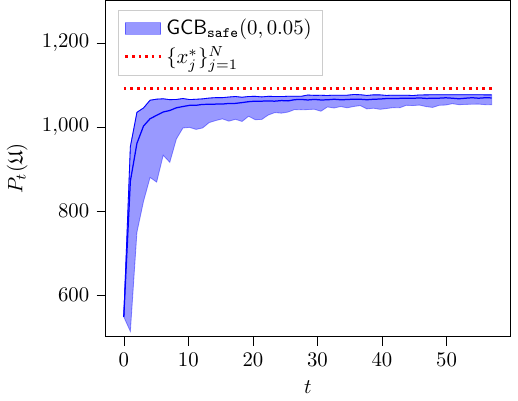}\label{fig:addx7}}}
	\subfloat[]{\scalebox{0.67}{\includegraphics{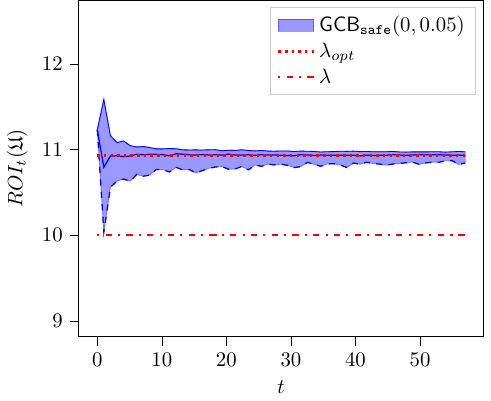}\label{fig:addx8}}}
	\subfloat[]{\scalebox{0.67}{\includegraphics{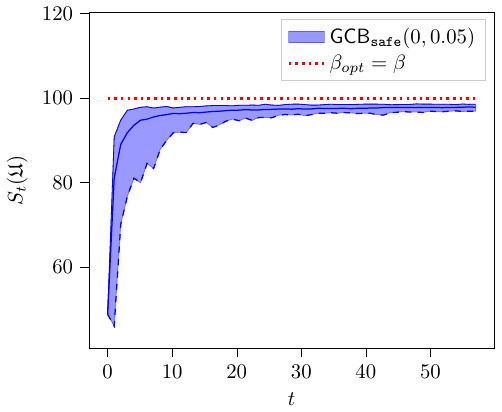}\label{fig:addx9}}}	
	\caption{Results of Experiment \#3: daily revenue (a), ROI (b), and spend (c) obtained by and \alg{}$(0,0.05)$. The dash-dotted lines correspond to the optimum values for the revenue and ROI, while the dashed lines correspond to the values of the ROI and budget constraints.}
	\label{fig:add_exp2_2}
\end{figure*}

\begin{figure*}[th!]
	\centering
	\captionsetup[subfigure]{oneside,margin={1.05cm,0cm}}
	\subfloat[]{\scalebox{0.67}{\includegraphics{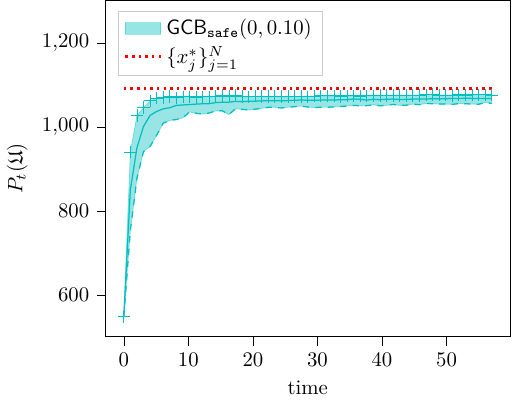}\label{fig:addx10}}}
	\subfloat[]{\scalebox{0.67}{\includegraphics{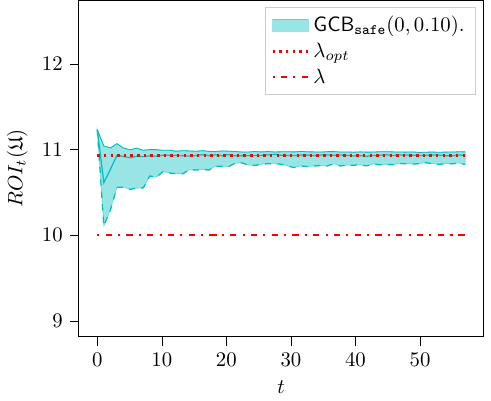}\label{fig:addx11}}}
	\subfloat[]{\scalebox{0.67}{\includegraphics{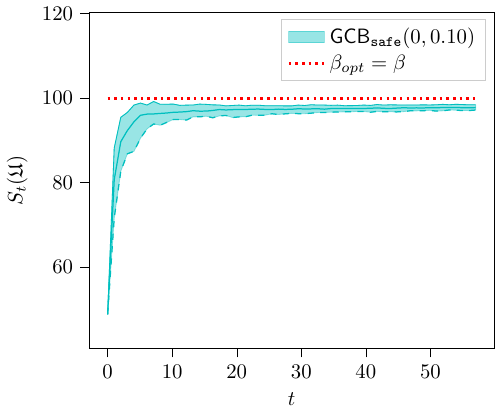}\label{fig:addx12}}}	
	\caption{Results of Experiment \#3: daily revenue (a), ROI (b), and spend (c) obtained by and \alg{}$(0,0.10)$. The dash-dotted lines correspond to the optimum values for the revenue and ROI, while the dashed lines correspond to the values of the ROI and budget constraints.}
	\label{fig:add_exp2_3}
\end{figure*}

\begin{figure*}[th!]
	\centering
	\captionsetup[subfigure]{oneside,margin={1.05cm,0cm}}
	\subfloat[]{\scalebox{0.67}{\includegraphics{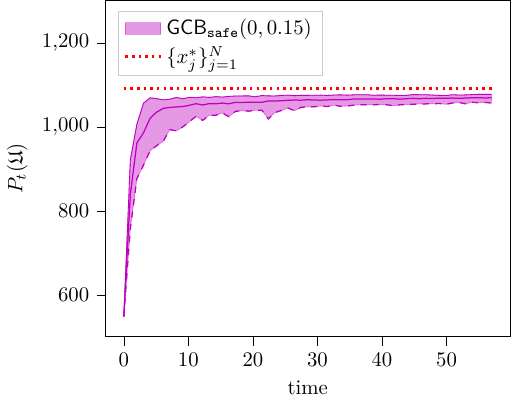}\label{fig:addx13}}}
	\subfloat[]{\scalebox{0.67}{\includegraphics{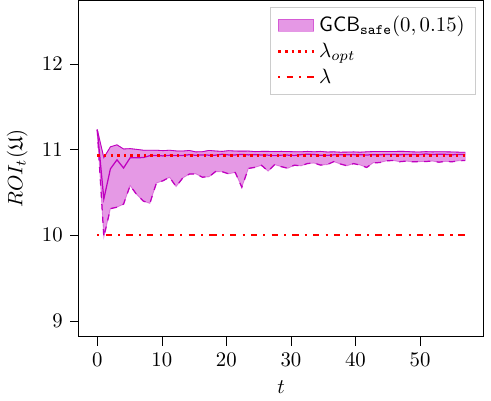}\label{fig:addx14}}}
	\subfloat[]{\scalebox{0.67}{\includegraphics{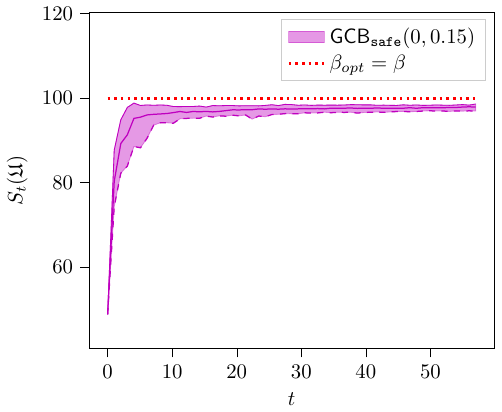}\label{fig:addx15}}}	
	\caption{Results of Experiment \#3: daily revenue (a), ROI (b), and spend (c) obtained by and \alg{}$(0,0.15)$. The dash-dotted lines correspond to the optimum values for the revenue and ROI, while the dashed lines correspond to the values of the ROI and budget constraints.}
	\label{fig:add_exp2_4}
\end{figure*}


\begin{thebibliography}{38}
	
	
	\ifx \showCODEN    \undefined \def \showCODEN     #1{\unskip}     \fi
	\ifx \showDOI      \undefined \def \showDOI       #1{#1}\fi
	\ifx \showISBNx    \undefined \def \showISBNx     #1{\unskip}     \fi
	\ifx \showISBNxiii \undefined \def \showISBNxiii  #1{\unskip}     \fi
	\ifx \showISSN     \undefined \def \showISSN      #1{\unskip}     \fi
	\ifx \showLCCN     \undefined \def \showLCCN      #1{\unskip}     \fi
	\ifx \shownote     \undefined \def \shownote      #1{#1}          \fi
	\ifx \showarticletitle \undefined \def \showarticletitle #1{#1}   \fi
	\ifx \showURL      \undefined \def \showURL       {\relax}        \fi
	\providecommand\bibfield[2]{#2}
	\providecommand\bibinfo[2]{#2}
	\providecommand\natexlab[1]{#1}
	\providecommand\showeprint[2][]{arXiv:#2}
	
	\bibitem[Accabi et~al\mbox{.}(2018)]%
	{accabi2018gaussian}
	\bibfield{author}{\bibinfo{person}{G.~M. Accabi}, \bibinfo{person}{F. Trov\`o},
		\bibinfo{person}{A. Nuara}, \bibinfo{person}{N. Gatti}, {and}
		\bibinfo{person}{M. Restelli}.} \bibinfo{year}{2018}\natexlab{}.
	\newblock \showarticletitle{When Gaussian Processes Meet Combinatorial Bandits:
		GCB}. In \bibinfo{booktitle}{\emph{{EWRL}}}. \bibinfo{pages}{1--11}.
	\newblock
	
	
	\bibitem[Amani et~al\mbox{.}(2020)]%
	{amani2020regret}
	\bibfield{author}{\bibinfo{person}{S. Amani}, \bibinfo{person}{M. Alizadeh},
		{and} \bibinfo{person}{C. Thrampoulidis}.} \bibinfo{year}{2020}\natexlab{}.
	\newblock \showarticletitle{Regret Bound for Safe Gaussian Process Bandit
		Optimization}. In \bibinfo{booktitle}{\emph{{L4DC}}}.
	\bibinfo{pages}{158--159}.
	\newblock
	
	
	\bibitem[Badanidiyuru et~al\mbox{.}(2018)]%
	{badanidiyuru2018bandits}
	\bibfield{author}{\bibinfo{person}{Ashwinkumar Badanidiyuru},
		\bibinfo{person}{Robert Kleinberg}, {and} \bibinfo{person}{Aleksandrs
			Slivkins}.} \bibinfo{year}{2018}\natexlab{}.
	\newblock \showarticletitle{Bandits with knapsacks}.
	\newblock \bibinfo{journal}{\emph{Journal of the ACM (JACM)}}
	\bibinfo{volume}{65}, \bibinfo{number}{3} (\bibinfo{year}{2018}),
	\bibinfo{pages}{1--55}.
	\newblock
	
	
	\bibitem[Balseiro et~al\mbox{.}(2020)]%
	{balseiro2020dual}
	\bibfield{author}{\bibinfo{person}{S.~Balseiro}, \bibinfo{person}{H~Lu}, {and} \bibinfo{person}{V.~Mirrokni}.}
	\bibinfo{year}{2020}\natexlab{}.
	\newblock \showarticletitle{Dual mirror descent for online allocation
		problems}. In \bibinfo{booktitle}{\emph{{ICML}}}. \bibinfo{pages}{613--628}.
	\newblock
	
	
	\bibitem[Balseiro and Gur(2019)]%
	{balseiro2019learning}
	\bibfield{author}{\bibinfo{person}{S.~R Balseiro} {and}
		\bibinfo{person}{Y.~Gur}.} \bibinfo{year}{2019}\natexlab{}.
	\newblock \showarticletitle{Learning in repeated auctions with budgets: Regret
		minimization and equilibrium}.
	\newblock \bibinfo{journal}{\emph{MAN SCIE}} \bibinfo{volume}{65},
	\bibinfo{number}{9} (\bibinfo{year}{2019}), \bibinfo{pages}{3952--3968}.
	\newblock
	
	
	\bibitem[Borgs et~al\mbox{.}(2007)]%
	{borgs2007dynamics}
	\bibfield{author}{\bibinfo{person}{C. Borgs}, \bibinfo{person}{J. Chayes},
		\bibinfo{person}{N. Immorlica}, \bibinfo{person}{K. Jain},
		\bibinfo{person}{O. Etesami}, {and} \bibinfo{person}{M. Mahdian}.}
	\bibinfo{year}{2007}\natexlab{}.
	\newblock \showarticletitle{Dynamics of bid optimization in online
		advertisement auctions}. In \bibinfo{booktitle}{\emph{{WWW}}}.
	\bibinfo{pages}{531--540}.
	\newblock
	
	
	\bibitem[Cao and Liu(2019)]%
	{cao2019online}
	\bibfield{author}{\bibinfo{person}{X. Cao} {and} \bibinfo{person}{K.~J.~Ray
			Liu}.} \bibinfo{year}{2019}\natexlab{}.
	\newblock \showarticletitle{Online Convex Optimization With Time-Varying
		Constraints and Bandit Feedback}.
	\newblock \bibinfo{journal}{\emph{IEEE T AUTOMAT CONTR}} \bibinfo{volume}{64},
	\bibinfo{number}{7} (\bibinfo{year}{2019}), \bibinfo{pages}{2665--2680}.
	\newblock
	
	
	\bibitem[Chen et~al\mbox{.}(2013)]%
	{chen2013combinatorial}
	\bibfield{author}{\bibinfo{person}{W. Chen}, \bibinfo{person}{Y. Wang}, {and}
		\bibinfo{person}{Y. Yuan}.} \bibinfo{year}{2013}\natexlab{}.
	\newblock \showarticletitle{Combinatorial multi-armed bandit: General framework
		and applications}. In \bibinfo{booktitle}{\emph{{ICML}}}.
	\bibinfo{pages}{151--159}.
	\newblock
	
	
	\bibitem[Chowdhury and Gopalan(2017)]%
	{chowdhury2017kernelized}
	\bibfield{author}{\bibinfo{person}{S.R. Chowdhury} {and} \bibinfo{person}{A.
			Gopalan}.} \bibinfo{year}{2017}\natexlab{}.
	\newblock \showarticletitle{On kernelized multi-armed bandits}. In
	\bibinfo{booktitle}{\emph{{ICML}}}. \bibinfo{pages}{844--853}.
	\newblock
	
	
	\bibitem[Deng et~al\mbox{.}(2023)]%
	{deng2023multi}
	\bibfield{author}{\bibinfo{person}{Y.~Deng}, \bibinfo{person}{N.~Golrezaei}, \bibinfo{person}{P.~Jaillet}, \bibinfo{person}{J.~Cheuk~Nam Liang}, {and} \bibinfo{person}{V.~Mirrokni}.}
	\bibinfo{year}{2023}\natexlab{}.
	\newblock \showarticletitle{Multi-channel autobidding with budget and roi
		constraints}. In \bibinfo{booktitle}{\emph{{ICML}}}.
	\bibinfo{pages}{7617--7644}.
	\newblock
	
	
	\bibitem[Devanur and Kakade(2009)]%
	{devanur2009price}
	\bibfield{author}{\bibinfo{person}{N.~R. Devanur} {and} \bibinfo{person}{S.~M.
			Kakade}.} \bibinfo{year}{2009}\natexlab{}.
	\newblock \showarticletitle{The price of truthfulness for pay-per-click
		auctions}. In \bibinfo{booktitle}{\emph{{ACM EC}}}. \bibinfo{pages}{99--106}.
	\newblock
	
	
	\bibitem[Ding et~al\mbox{.}(2013)]%
	{ding2013multi}
	\bibfield{author}{\bibinfo{person}{W. Ding}, \bibinfo{person}{T. Qin},
		\bibinfo{person}{X.-D. Zhang}, {and} \bibinfo{person}{T.Y. Liu}.}
	\bibinfo{year}{2013}\natexlab{}.
	\newblock \showarticletitle{Multi-Armed Bandit with Budget Constraint and
		Variable Costs}. In \bibinfo{booktitle}{\emph{{AAAI}}}.
	\bibinfo{pages}{232--238}.
	\newblock
	
	
	\bibitem[Feldman et~al\mbox{.}(2007)]%
	{feldman2007budget}
	\bibfield{author}{\bibinfo{person}{J. Feldman}, \bibinfo{person}{S.
			Muthukrishnan}, \bibinfo{person}{M. Pal}, {and} \bibinfo{person}{C. Stein}.}
	\bibinfo{year}{2007}\natexlab{}.
	\newblock \showarticletitle{Budget optimization in search-based advertising
		auctions}. In \bibinfo{booktitle}{\emph{{EC}}}. \bibinfo{pages}{40--49}.
	\newblock
	
	
	\bibitem[Feng et~al\mbox{.}(2023)]%
	{feng2023online}
	\bibfield{author}{\bibinfo{person}{Zhe Feng}, \bibinfo{person}{Swati
			Padmanabhan}, {and} \bibinfo{person}{Di Wang}.}
	\bibinfo{year}{2023}\natexlab{}.
	\newblock \showarticletitle{Online Bidding Algorithms for Return-on-Spend
		Constrained Advertisers}. In \bibinfo{booktitle}{\emph{{WWW}}}.
	\bibinfo{pages}{3550--3560}.
	\newblock
	
	
	\bibitem[Galichet et~al\mbox{.}(2013)]%
	{galichet2013exploration}
	\bibfield{author}{\bibinfo{person}{N. Galichet}, \bibinfo{person}{M. Sebag},
		{and} \bibinfo{person}{O. Teytaud}.} \bibinfo{year}{2013}\natexlab{}.
	\newblock \showarticletitle{Exploration vs exploitation vs safety: Risk-aware
		multi-armed bandits}. In \bibinfo{booktitle}{\emph{{ACML}}}.
	\bibinfo{pages}{245--260}.
	\newblock
	
	
	\bibitem[Garcia and Fern{\'a}ndez(2012)]%
	{garcia2012safe}
	\bibfield{author}{\bibinfo{person}{J. Garcia} {and} \bibinfo{person}{F.
			Fern{\'a}ndez}.} \bibinfo{year}{2012}\natexlab{}.
	\newblock \showarticletitle{Safe exploration of state and action spaces in
		reinforcement learning}.
	\newblock \bibinfo{journal}{\emph{J ARTIF INTELL RES}}  \bibinfo{volume}{45}
	(\bibinfo{year}{2012}), \bibinfo{pages}{515--564}.
	\newblock
	
	
	\bibitem[Golrezaei et~al\mbox{.}(2021a)]%
	{golrezaei2021bidding}
	\bibfield{author}{\bibinfo{person}{Negin Golrezaei}, \bibinfo{person}{Patrick
			Jaillet}, \bibinfo{person}{Jason Cheuk~Nam Liang}, {and}
		\bibinfo{person}{Vahab Mirrokni}.} \bibinfo{year}{2021}\natexlab{a}.
	\newblock \showarticletitle{Bidding and pricing in budget and roi constrained
		markets}.
	\newblock \bibinfo{journal}{\emph{arXiv preprint arXiv:2107.07725}}
	(\bibinfo{year}{2021}).
	\newblock
	
	
	\bibitem[Golrezaei et~al\mbox{.}(2021b)]%
	{golrezaei2021auction}
	\bibfield{author}{\bibinfo{person}{Negin Golrezaei}, \bibinfo{person}{Ilan
			Lobel}, {and} \bibinfo{person}{Renato Paes~Leme}.}
	\bibinfo{year}{2021}\natexlab{b}.
	\newblock \showarticletitle{Auction design for ROI-constrained buyers}. In
	\bibinfo{booktitle}{\emph{{WWW}}}. \bibinfo{pages}{3941--3952}.
	\newblock
	
	
	\bibitem[Hans et~al\mbox{.}(2008)]%
	{hans2008safe}
	\bibfield{author}{\bibinfo{person}{A. Hans}, \bibinfo{person}{D.
			Schneega{\ss}}, \bibinfo{person}{A.~M. Sch{\"a}fer}, {and}
		\bibinfo{person}{S. Udluft}.} \bibinfo{year}{2008}\natexlab{}.
	\newblock \showarticletitle{Safe exploration for reinforcement learning.}. In
	\bibinfo{booktitle}{\emph{{ESANN}}}. \bibinfo{pages}{143--148}.
	\newblock
	
	
	\bibitem[IAB(2024)]%
	{iab2024iab}
	\bibfield{author}{\bibinfo{person}{IAB}.} \bibinfo{year}{2024}\natexlab{}.
	\newblock \bibinfo{title}{{Interactive Advertising Bureau (IAB)} internet
		advertising revenue report, {Full year} $2023$ results}.
	\newblock
	\bibinfo{howpublished}{\url{https://www.iab.com/wp-content/uploads/2024/04/IAB_PwC_Internet_Ad_Revenue_Report_2024.pdf}}.
	\newblock
	\newblock
	\shownote{Online; accessed 06 August 2024}.
	
	
	\bibitem[Kong et~al\mbox{.}(2018)]%
	{kong2018combinatorial}
	\bibfield{author}{\bibinfo{person}{D. Kong}, \bibinfo{person}{X. Fan},
		\bibinfo{person}{K. Shmakov}, {and} \bibinfo{person}{J. Yang}.}
	\bibinfo{year}{2018}\natexlab{}.
	\newblock \showarticletitle{A Combinational Optimization Approach for
		Advertising Budget Allocation}. In \bibinfo{booktitle}{\emph{{WWW}}}.
	\bibinfo{pages}{53--54}.
	\newblock
	
	
	\bibitem[Kullback and Leibler(1951)]%
	{kullback1951information}
	\bibfield{author}{\bibinfo{person}{Solomon Kullback} {and}
		\bibinfo{person}{Richard~A Leibler}.} \bibinfo{year}{1951}\natexlab{}.
	\newblock \showarticletitle{On information and sufficiency}.
	\newblock \bibinfo{journal}{\emph{The annals of mathematical statistics}}
	\bibinfo{volume}{22}, \bibinfo{number}{1} (\bibinfo{year}{1951}),
	\bibinfo{pages}{79--86}.
	\newblock
	
	
	\bibitem[Lucier et~al\mbox{.}(2024)]%
	{lucier2024autobidders}
	\bibfield{author}{\bibinfo{person}{Brendan Lucier}, \bibinfo{person}{Sarath
			Pattathil}, \bibinfo{person}{Aleksandrs Slivkins}, {and}
		\bibinfo{person}{Mengxiao Zhang}.} \bibinfo{year}{2024}\natexlab{}.
	\newblock \showarticletitle{Autobidders with budget and roi constraints:
		Efficiency, regret, and pacing dynamics}. In
	\bibinfo{booktitle}{\emph{{ALT}}}. \bibinfo{pages}{3642--3643}.
	\newblock
	
	
	\bibitem[Mannor et~al\mbox{.}(2009)]%
	{mannor2009online}
	\bibfield{author}{\bibinfo{person}{S. Mannor}, \bibinfo{person}{J.~N.
			Tsitsiklis}, {and} \bibinfo{person}{J.~Y. Yu}.}
	\bibinfo{year}{2009}\natexlab{}.
	\newblock \showarticletitle{Online Learning with Sample Path Constraints}.
	\newblock \bibinfo{journal}{\emph{J MACH LEARN RES}}  \bibinfo{volume}{10}
	(\bibinfo{year}{2009}), \bibinfo{pages}{569--590}.
	\newblock
	
	
	\bibitem[Moradipari et~al\mbox{.}(2020)]%
	{moradipari2020stagewise}
	\bibfield{author}{\bibinfo{person}{A. Moradipari}, \bibinfo{person}{C.
			Thrampoulidis}, {and} \bibinfo{person}{M. Alizadeh}.}
	\bibinfo{year}{2020}\natexlab{}.
	\newblock \showarticletitle{Stage-wise Conservative Linear Bandits}. In
	\bibinfo{booktitle}{\emph{{NeurIPS}}}. \bibinfo{pages}{11191--11201}.
	\newblock
	
	
	\bibitem[Nuara et~al\mbox{.}(2019)]%
	{nuara2019dealing}
	\bibfield{author}{\bibinfo{person}{A. Nuara}, \bibinfo{person}{N. Sosio},
		\bibinfo{person}{F. Trov\`o}, \bibinfo{person}{M.~C. Zaccardi},
		\bibinfo{person}{N. Gatti}, {and} \bibinfo{person}{M. Restelli}.}
	\bibinfo{year}{2019}\natexlab{}.
	\newblock \showarticletitle{Dealing with Interdependencies and Uncertainty in
		Multi-Channel Advertising Campaigns Optimization}. In
	\bibinfo{booktitle}{\emph{{WWW}}}. \bibinfo{pages}{1376–1386}.
	\newblock
	
	
	\bibitem[Nuara et~al\mbox{.}(2018)]%
	{nuara2018combinatorial}
	\bibfield{author}{\bibinfo{person}{A. Nuara}, \bibinfo{person}{F. Trov{\`{o}}},
		\bibinfo{person}{N. Gatti}, {and} \bibinfo{person}{M. Restelli}.}
	\bibinfo{year}{2018}\natexlab{}.
	\newblock \showarticletitle{A Combinatorial-Bandit Algorithm for the Online
		Joint Bid/Budget Optimization of Pay-per-Click Advertising Campaigns}. In
	\bibinfo{booktitle}{\emph{{AAAI}}}. \bibinfo{pages}{2379--2386}.
	\newblock
	
	
	\bibitem[Nuara et~al\mbox{.}(2022)]%
	{nuara2022online}
	\bibfield{author}{\bibinfo{person}{Alessandro Nuara},
		\bibinfo{person}{Francesco Trov{\`o}}, \bibinfo{person}{Nicola Gatti}, {and}
		\bibinfo{person}{Marcello Restelli}.} \bibinfo{year}{2022}\natexlab{}.
	\newblock \showarticletitle{Online joint bid/daily budget optimization of
		internet advertising campaigns}.
	\newblock \bibinfo{journal}{\emph{Artificial Intelligence}}
	\bibinfo{volume}{305} (\bibinfo{year}{2022}), \bibinfo{pages}{103663}.
	\newblock
	
	
	\bibitem[Pirotta et~al\mbox{.}(2013)]%
	{pirotta2013safe}
	\bibfield{author}{\bibinfo{person}{M. Pirotta}, \bibinfo{person}{M. Restelli},
		\bibinfo{person}{A. Pecorino}, {and} \bibinfo{person}{D. Calandriello}.}
	\bibinfo{year}{2013}\natexlab{}.
	\newblock \showarticletitle{Safe policy iteration}. In
	\bibinfo{booktitle}{\emph{{ICML}}}. \bibinfo{pages}{307--315}.
	\newblock
	
	
	\bibitem[Rasmussen and Williams(2006)]%
	{rasmussen2006gaussian}
	\bibfield{author}{\bibinfo{person}{C.~E. Rasmussen} {and}
		\bibinfo{person}{C.~K. Williams}.} \bibinfo{year}{2006}\natexlab{}.
	\newblock \bibinfo{booktitle}{\emph{Gaussian processes for machine learning}}.
	Vol.~\bibinfo{volume}{1}.
	\newblock \bibinfo{publisher}{MIT Press}.
	\newblock
	
	
	\bibitem[Sankararaman and Slivkins(2018)]%
	{sankararaman2018combinatorial}
	\bibfield{author}{\bibinfo{person}{Karthik~Abinav Sankararaman} {and}
		\bibinfo{person}{Aleksandrs Slivkins}.} \bibinfo{year}{2018}\natexlab{}.
	\newblock \showarticletitle{Combinatorial semi-bandits with knapsacks}. In
	\bibinfo{booktitle}{\emph{{AISTATS}}}. \bibinfo{pages}{1760--1770}.
	\newblock
	
	
	\bibitem[Srinivas et~al\mbox{.}(2010)]%
	{srinivas2010gaussian}
	\bibfield{author}{\bibinfo{person}{N. Srinivas}, \bibinfo{person}{A. Krause},
		\bibinfo{person}{M. Seeger}, {and} \bibinfo{person}{S.~M. Kakade}.}
	\bibinfo{year}{2010}\natexlab{}.
	\newblock \showarticletitle{Gaussian Process Optimization in the Bandit
		Setting: No Regret and Experimental Design}. In
	\bibinfo{booktitle}{\emph{{ICML}}}. \bibinfo{pages}{1015--1022}.
	\newblock
	
	
	\bibitem[Sui et~al\mbox{.}(2015)]%
	{sui2015safe}
	\bibfield{author}{\bibinfo{person}{Y. Sui}, \bibinfo{person}{A. Gotovos},
		\bibinfo{person}{J. Burdick}, {and} \bibinfo{person}{A. Krause}.}
	\bibinfo{year}{2015}\natexlab{}.
	\newblock \showarticletitle{Safe exploration for optimization with Gaussian
		processes}. In \bibinfo{booktitle}{\emph{{ICML}}}.
	\bibinfo{pages}{997--1005}.
	\newblock
	
	
	\bibitem[Szymanski and Lee(2006)]%
	{szymanski2006impact}
	\bibfield{author}{\bibinfo{person}{B.~K. Szymanski} {and} \bibinfo{person}{J.
			Lee}.} \bibinfo{year}{2006}\natexlab{}.
	\newblock \showarticletitle{Impact of roi on bidding and revenue in sponsored
		search advertisement auctions}. In \bibinfo{booktitle}{\emph{Workshop on
			Sponsored Search Auctions}}, Vol.~\bibinfo{volume}{1}. \bibinfo{pages}{1--8}.
	\newblock
	
	
	\bibitem[Thomaidou et~al\mbox{.}(2014)]%
	{thomaidou2014toward}
	\bibfield{author}{\bibinfo{person}{S. Thomaidou}, \bibinfo{person}{K.
			Liakopoulos}, {and} \bibinfo{person}{M. Vazirgiannis}.}
	\bibinfo{year}{2014}\natexlab{}.
	\newblock \showarticletitle{Toward an integrated framework for automated
		development and optimization of online advertising campaigns}.
	\newblock \bibinfo{journal}{\emph{INTELL DATA ANAL}} \bibinfo{volume}{18},
	\bibinfo{number}{6} (\bibinfo{year}{2014}), \bibinfo{pages}{1199--1227}.
	\newblock
	
	
	\bibitem[Trov{\`o} et~al\mbox{.}(2016)]%
	{trovo2016budgeted}
	\bibfield{author}{\bibinfo{person}{F. Trov{\`o}}, \bibinfo{person}{S.
			Paladino}, \bibinfo{person}{M. Restelli}, {and} \bibinfo{person}{N. Gatti}.}
	\bibinfo{year}{2016}\natexlab{}.
	\newblock \showarticletitle{Budgeted Multi-Armed Bandit in Continuous Action
		Space}. In \bibinfo{booktitle}{\emph{{ECAI}}}. \bibinfo{pages}{560--568}.
	\newblock
	
	
	\bibitem[Vazirani et~al\mbox{.}(2007)]%
	{vazirani2007algorithmic}
	\bibfield{author}{\bibinfo{person}{V.~V. Vazirani}, \bibinfo{person}{N. Nisan},
		\bibinfo{person}{T. Roughgarden}, {and} \bibinfo{person}{E. Tardos}.}
	\bibinfo{year}{2007}\natexlab{}.
	\newblock \bibinfo{booktitle}{\emph{Algorithmic Game Theory}}.
	\newblock \bibinfo{publisher}{Cambridge University Press}.
	\newblock
	
	
	\bibitem[Zhang et~al\mbox{.}(2012)]%
	{zhang2012joint}
	\bibfield{author}{\bibinfo{person}{W. Zhang}, \bibinfo{person}{Y. Zhang},
		\bibinfo{person}{B. Gao}, \bibinfo{person}{Y. Yu}, \bibinfo{person}{X. Yuan},
		{and} \bibinfo{person}{T.-Y. Liu}.} \bibinfo{year}{2012}\natexlab{}.
	\newblock \showarticletitle{Joint optimization of bid and budget allocation in
		sponsored search}. In \bibinfo{booktitle}{\emph{{SIGKDD}}}.
	\bibinfo{pages}{1177--1185}.
	\newblock

\end{thebibliography}
\end{document}